\documentclass{article}

\usepackage[nonatbib, preprint]{neurips_2022}
\usepackage[numbers]{natbib}

\usepackage[utf8]{inputenc} 
\usepackage[T1]{fontenc}    
\usepackage[colorlinks,
            linkcolor=green,
            citecolor=blue,
]{hyperref}
\usepackage{url}            
\usepackage{booktabs}       
\usepackage{amsfonts}       
\usepackage{nicefrac}       
\usepackage{microtype}      
\usepackage{xcolor}         

\usepackage{algorithm}
\usepackage{algpseudocode} 

\usepackage{graphicx}
\usepackage{subcaption}
\usepackage{amsmath}
\usepackage{mathtools}
\usepackage{amsthm,amssymb}
\usepackage{mathrsfs}
\usepackage{bm}
\usepackage{multicol,lipsum}
\usepackage{amsmath}
\usepackage{amssymb}
\usepackage{mathtools}
\usepackage{amsthm}
\usepackage{diagbox}
\usepackage{wrapfig}

\usepackage[capitalize,noabbrev]{cleveref}

\newtheorem{theorem}{Theorem}
\newtheorem{proposition}[theorem]{Proposition}

\newcommand{\inlineeqnum}{\refstepcounter{equation}~~\mbox{(\theequation)}}

\title{Evolutionary Action Selection \\ for Gradient-based Policy Learning}

\author{%
  Yan Ma \\
  Fudan University\\
  Shanghai, China\\
  \texttt{20210860024@fudan.edu.cn} \\
  \And
  Tianxing Liu \\
  Fudan University\\
  Shanghai, China\\
  \texttt{liutianxing@126.com} \\
  \And
  Bingsheng Wei \\
  Fudan University\\
  Shanghai, China\\
  \texttt{weibingsheng@fudan.edu.cn} \\
  \And
  Yi Liu \\
  Fudan University\\
  Shanghai, China\\
  \texttt{liuyi\_@fudan.edu.cn} \\
  \And
  Kang Xu \\
  Fudan University\\
  Shanghai, China\\
  \texttt{21210860020@m.fudan.edu.cn} \\
  \And
  Wei Li\\
  Fudan University\\
  Shanghai, China\\
  \texttt{fd\_liwei@fudan.edu.cn} \\
}

\begin{document}

\maketitle

\begin{abstract}
Evolutionary Algorithms (EAs) and Deep Reinforcement Learning (DRL) have recently been integrated to take the advantage of the both methods for better exploration and exploitation.
The evolutionary part in these hybrid methods maintains a population of policy networks.
However, existing methods focus on optimizing the parameters of policy network, which is usually high-dimensional and tricky for EA.
In this paper, we shift the target of evolution from high-dimensional parameter space to low-dimensional action space.
We propose Evolutionary Action Selection-Twin Delayed Deep Deterministic Policy Gradient (EAS-TD3), a novel hybrid method of EA and DRL.
In EAS, we focus on optimizing the action chosen by the policy network and attempt to obtain high-quality actions to promote policy learning through an evolutionary algorithm.
We conduct several experiments on challenging continuous control tasks.
The result shows that EAS-TD3 shows superior performance over other state-of-art methods.
\end{abstract}

\section{Introduction}
\label{sec:intro}

Deep Reinforcement Learning (DRL) has achieved impressive performance in Go~\citep{c:Go}, Atari games~\citep{c:Atari}, and continuous control tasks~\citep{rllab, SAC_application}. 
The purpose of DRL is to train an optimal, or nearly-optimal policy network that maximizes the reward function or other user-provided reinforcement signal. Recently, Evolutionary Algorithms (EA) have been applied to search the parameter space of neural networks to train policies and showed competitive results as DRL~\cite{ES, neurevolution}.
Some work compares DRL and EA~\cite{DRLvsES, PPO}, highlighting their respective pros and cons.
However, an emerging research direction tends to integrate them to promote each other for better policy search~\cite{erl_survey, erl_review}, in that they have several complementary properties and can be integrated together to benefit from the best of both worlds.
As one of the pioneers, \citet{ERL} proposed Evolutionary Reinforcement Learning (ERL), which has shown promising results in continuous control tasks.
ERL maintains a population containing $n$ policy networks trained by EA and one policy network trained by DRL.
The policy population provides a large number of samples $(s,a,r,s')$ to the RL policy and RL policy will be injected into the population periodically to replace the poorer individual.
With the help of EA's global optimization ability, ERL explores the parameter space of the policy network to search for outstanding solutions.
Most variants of the ERL~\cite{CEM-RL, CERL, PDERL, Super-RL, ESAC} follow the parameter space of the policy network as the target of evolution.
However, taking parameters of the policy network as the target of evolution may pose a potential problem.
That is, EA has to optimize a high dimensional parameter space, which is chanllenging for EA.
It is well known that EA is deficient in the optimization of high-dimensional spaces.
Although EA can be leveraged to train policy networks, it suffers from low sample efficiency since it is gradient-free.
Corresponding to ERL methods, the contribution of the evolutionary part will be limited.
In particular, we find that the evolutionary part may weaken the performance of the overall method at high dimensions of the parameter space.
As illustrated in ~\cref{fig:toy_c} (see more details in~\cref{sec:toy_example}), the performance of ERL method decays with the increase of policy parameters.
The problem in ERL methods is attributed to the use of EA to directly train the strategy network.
In a nutshell, taking the parameter space of the policy network as the target of evolution will limit the contribution of EA to such hybrid methods.
The motivation of this work is to transfer the evolutionary target from the high dimensional parameter space to a low dimensional space where EA is more proficient and avoids the problem mentioned earlier so as to better serve the best of both worlds.
\begin{figure}[htb]
 \centering
 \begin{subfigure}[t]{0.24\columnwidth}
     \centering
     \includegraphics[width=\columnwidth]{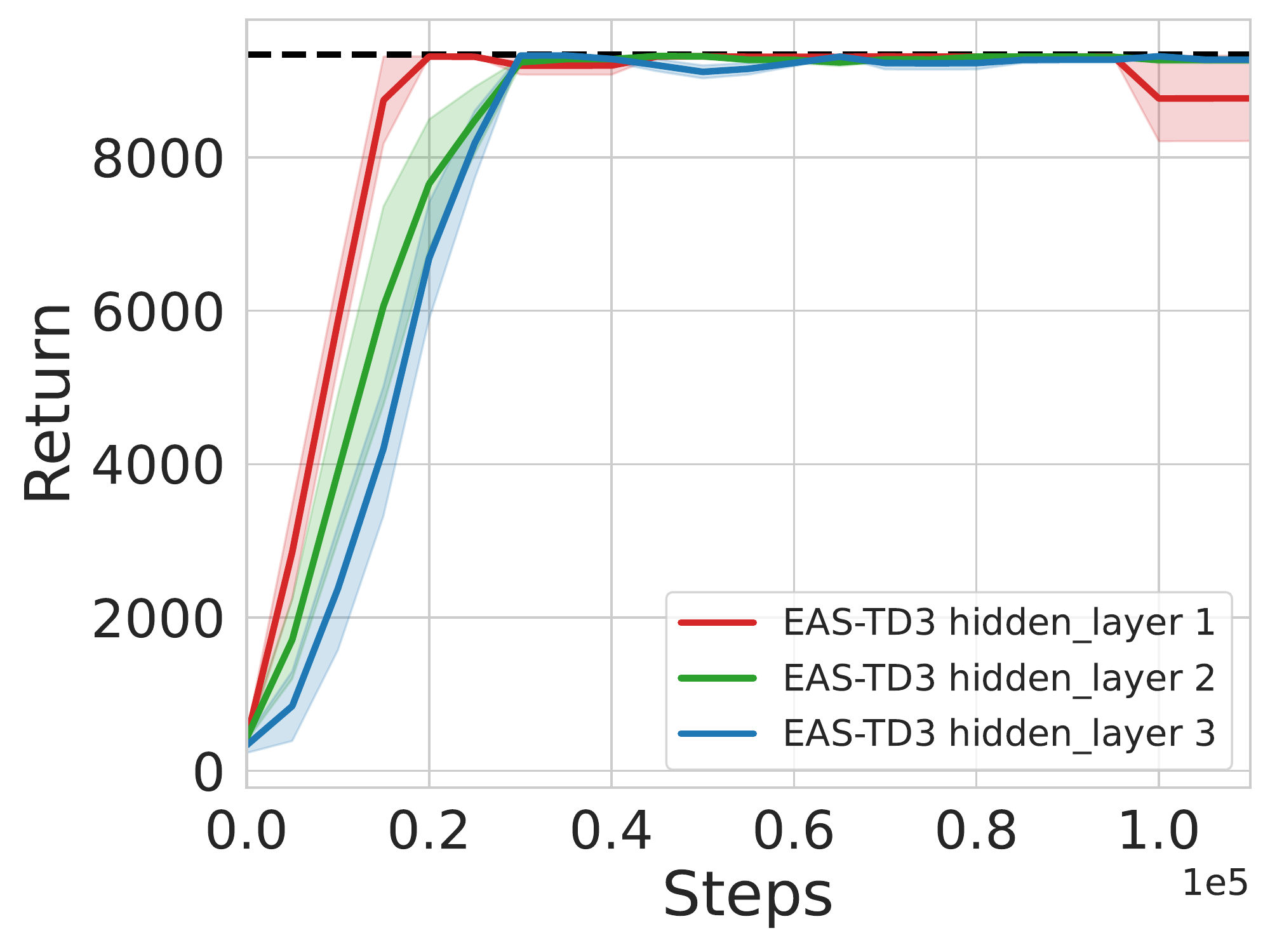}
     \vspace{-13pt}
     \caption{EAS-TD3}
     \label{fig:toy_a}
 \end{subfigure}%
 \hspace{0.1in}
 \begin{subfigure}[t]{0.24\columnwidth}
     \centering
     \includegraphics[width=\columnwidth]{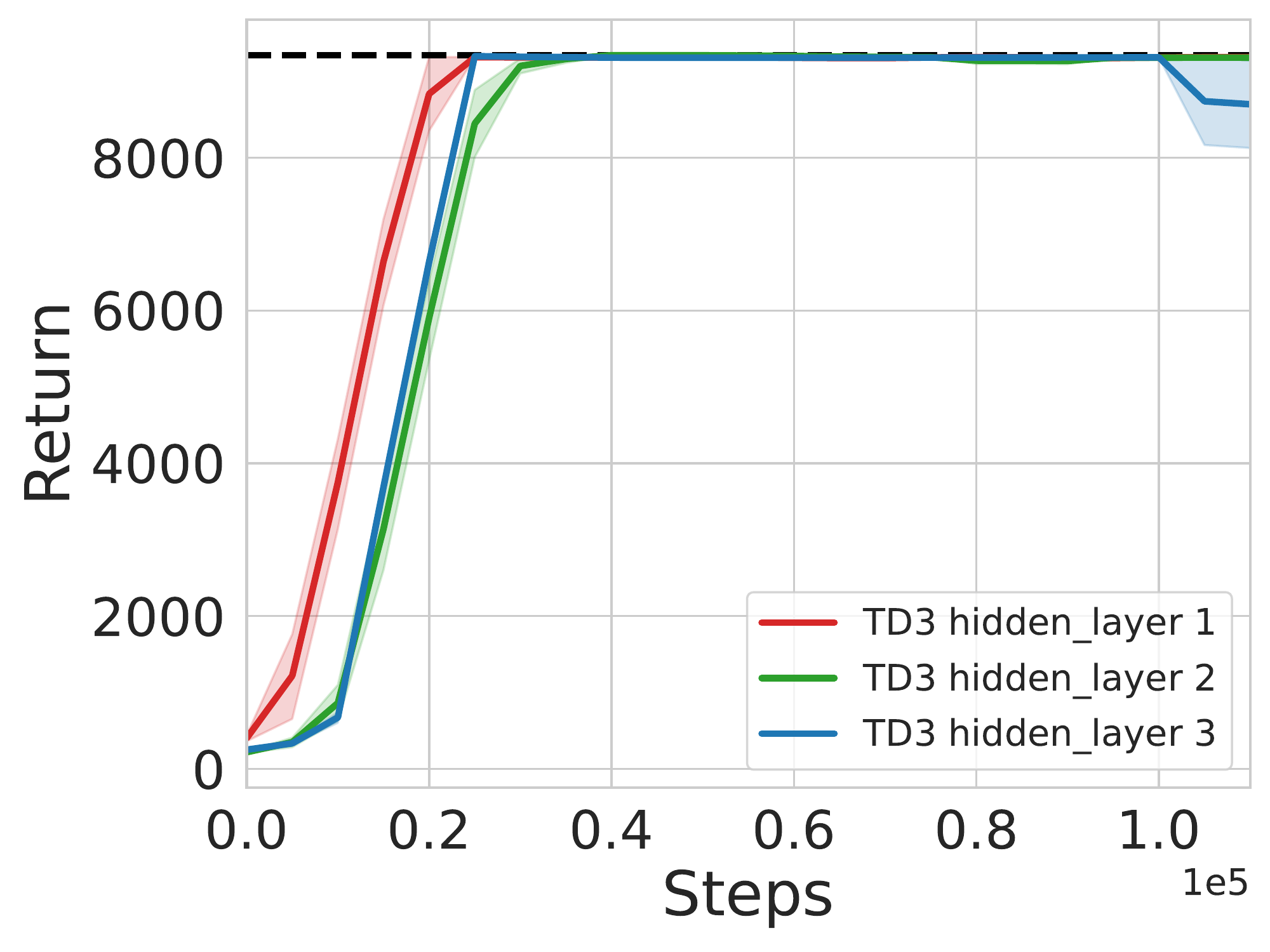}
     \vspace{-13pt}
     \caption{TD3}
     \label{fig:toy_b}
 \end{subfigure}%
 \hspace{0.1in}
 \begin{subfigure}[t]{0.24\columnwidth}
     \centering
     \includegraphics[width=\columnwidth]{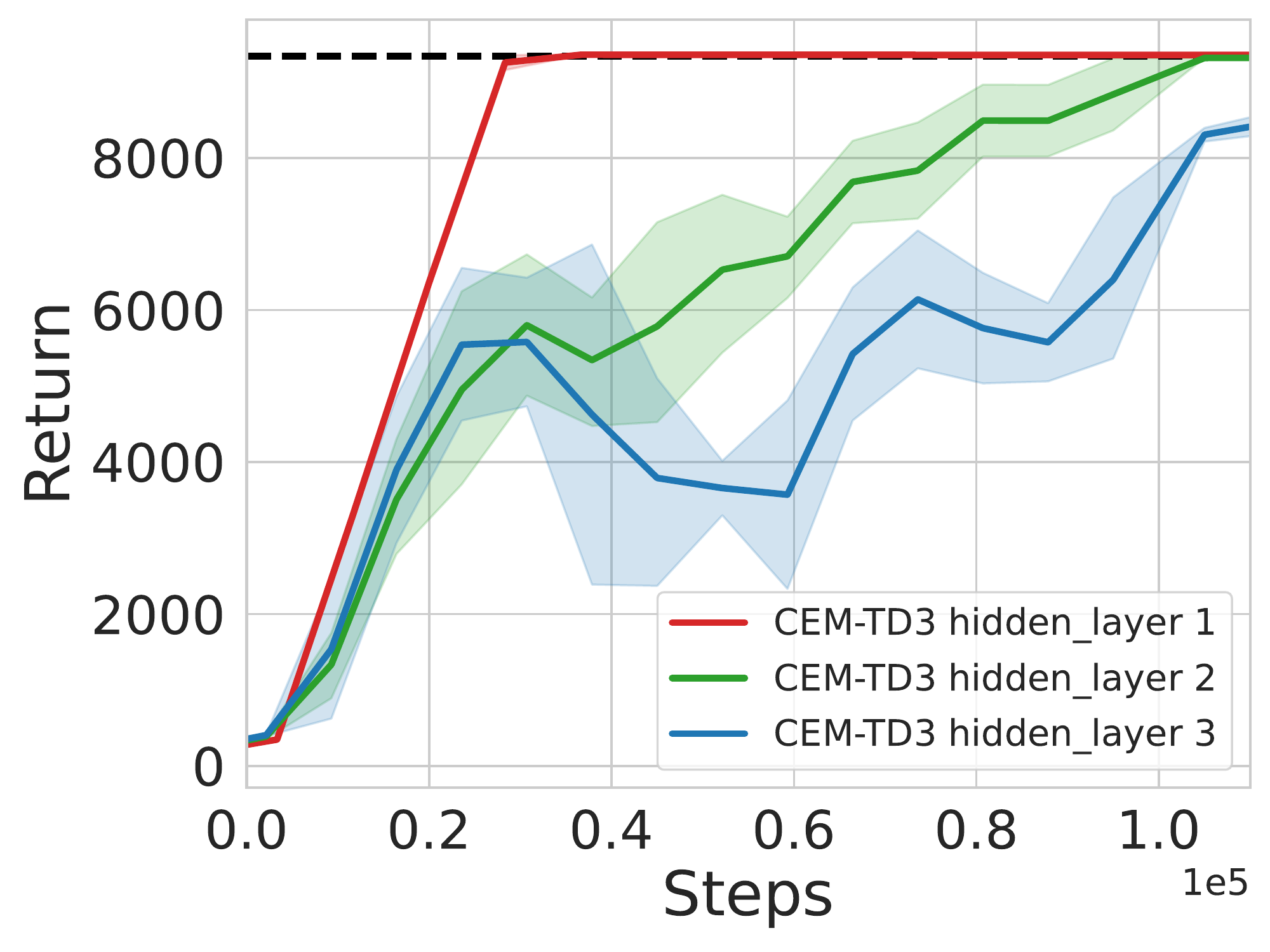}
     \vspace{-13pt}
     \caption{CEM-TD3}
     \label{fig:toy_c}
 \end{subfigure}%
\caption{Comparative Experiment of EAS-TD3 (our work), CEM-TD3 (one of the sota methods in ERL variants), TD3 on Inverted Double Pendulum task from Gym~\cite{gym} with different policy network parameters. 
The performance of CEM-TD3 decreases while that of EAS-TD3 remains stable.}
\label{fig:toy_example}
\vspace{-13pt}
\end{figure}

Our goal is unambiguous, which aims to train a good policy to maximize the cumulative reward of the given task.
A good policy $\mu(a|s)$ means it can choose a good action $a$ according to the state $s$ to maximize the cumulative reward.
The quality of the action largely determines the quality of the policy.
More importantly, the action space is relatively low dimensional and generally will not change for a given task.
Thus, we intend to take the action space as the target of evolution.
Concretely speaking, we leverage EA to optimize the action chosen by the RL policy and obtain better evolutionary actions.
The policy network trained by RL collects samples in the environment.
EA extracts actions from samples and evolves better evolutionary actions to promote the RL policy learning.
This maintains the flow of information between the EA part and RL part.

Based on the above insight, we propose Evolutionary Action Selection (EAS).
As shown in~\cref{fig:EAS}, EAS utilizes actions selected by RL policy to form a population and utilizes Particle Swarm Optimization (PSO)~\cite{PSO} to evolve the action population from generation to generation.
Finally, we obtain better evolutionary actions, which can promote the learning process through the evolutionary action gradient.
We choose PSO for two reasons.
Firstly, it has been widely used in many fields and proved to be effective in practice~\cite{PSOapplication}.
Secondly, it is easy to follow and will not add too much computational burden.
In addition, PSO can be replaced by other evolutionary algorithms, such as genetic algorithm~\cite{GA}, cross-entropy method~\cite{CEM}, and so on.

Our contributions are threefold:
(1) We empirically demonstrate that taking the policy network parameter space as the evolutionary target may lead to performance degradation.
(2) We transfer the target of evolution from high-dimensional parameter space to low-dimensional action space and propose a simple and effective mechanism to evolve action.
(3) We apply EAS to TD3~\cite{TD3} as EAS-TD3 and conduct a series of empirical studies on a benchmark suite of continuous control tasks to prove the feasibility and superiority of our approach.
\section{Related Work}\label{sec:2}
The idea of incorporating learning with evolution has been around for many years~\cite{EARL1992, EARL1999, EARL2006, EARL2011}.
With the brilliance of reinforcement learning, recent literature~\cite{GPO, GEP-PG, Super-RL} has begun to revisit the combination of the two to improve the performance of the overall approach.

As mentioned earlier, this paper is related to the recently proposed Evolutionary Reinforcement Learning (ERL)~\cite{ERL} framework.
ERL combines Genetic Algorithm (GA)~\cite{GA} with the off-policy DRL algorithm (DDPG)~\cite{DDPG} and incorporates the two processes to run concurrently formulating a framework.
Specifically, ERL maintains a policy population trained by EA and a policy network trained by RL.
By maintaining interactive information flow between EA and RL, the performance of the overall method is promoted.
The framework of ERL has triggered a variety of variants, which makes the efficient combination of EA and RL an emerging research direction for both the EA and RL community.
Collaborative Evolutionary Reinforcement Learning (CERL)~\cite{CERL} is the follow-up work of ERL. CERL attempts to train multiple policy networks with different hyperparameters to address the DRL's sensitivity to hyperparameters.
Moreover, Proximal Distilled Evolutionary Reinforcement Learning (PDERL)~\cite{PDERL} attempts to figure out the catastrophic forgetting of the neural network caused by the genetic operator used in ERL.
CEM-RL~\cite{CEM-RL} removes the single policy network trained by RL and instead allows half of the policy networks in the population to be trained directly by RL and the other half by cross entropy method (CEM)~\cite{CEM}.
This approach magnifies the impact of gradient-based policy learning methods on the evolutionary population, which improves the sample efficiency of the ERL framework.
AES-RL~\cite{AES-RL} proposes an efficient asynchronous method for integrating evolutionary and gradient-based policy search, which shortens the training time.
QD-PG~\cite{QD-RL} introduces Quality-Diversity (QD) algorithm for RL to address the problem of deceptive reward.
Note that there are some works in RL domain that embody the idea of action improvement~\cite{Qt-Opt, CGP,SAC-CEPO, GRAC, DAO}.
For instance, the QT-Opt algorithm ~\cite{Qt-Opt} randomly samples a batch of actions from the action space and leverages CEM to select the output action.
GRAC~\cite{GRAC} extends the idea of QT-Opt to stochastic policy and constructs two novel losses to make the critic network update more stable and robust.

Our approach focuses on integrating the advantages of EA and DRL to better learn policies.
We shift the target of evolution from high-dimensional policy network parameter space to low-dimensional action space and demonstrate empirically that evolving the action is a favorable alternative to evolving the parameters of the policy network.
%
\begin{figure*}[htb]
\centering
\includegraphics[width=0.85\textwidth]{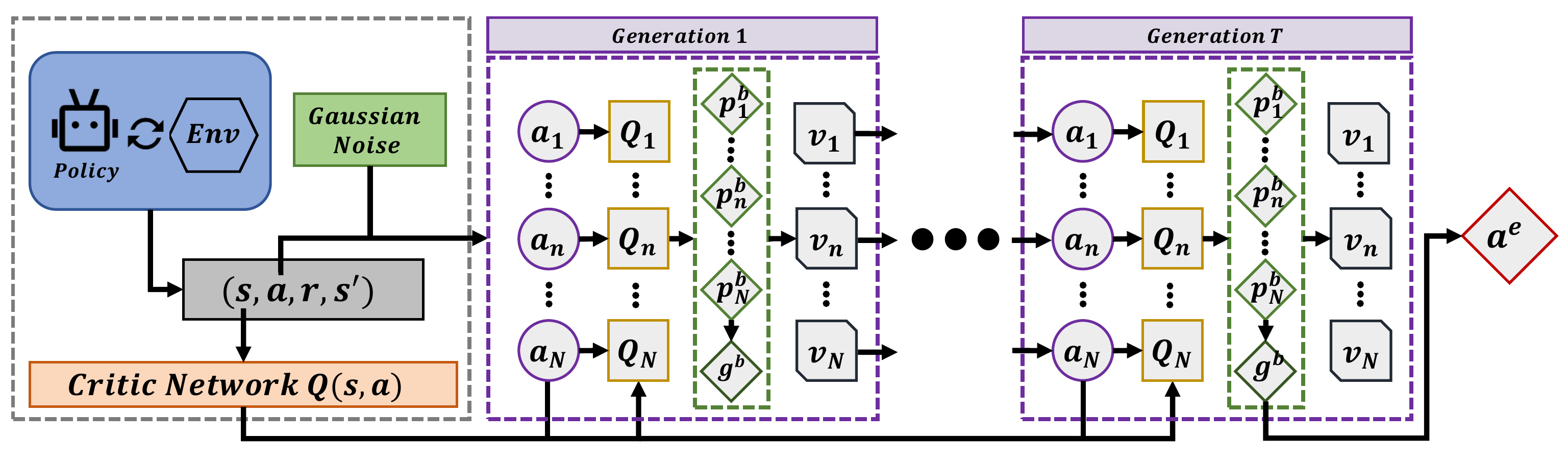}
\caption{The procedure of Evolutionary Action Selection.
Add Gaussian noise to the action in the sample and form the initial population.
The Gaussian noise we add is generally the white noise $\mathcal{N}(0,1)$.
As the fitness evaluator, the critic network is used to generate $Q$ values for the individual of the action population in each generation, from where we update the individual best action $p^b_n$ and global best action $g^b$.
Then, we update the action $a_n$ and the velocity $v_n$.
After several generations, the global best action $g^b$ will be served as the evolutionary action $a^e$.}
\label{fig:EAS}
\end{figure*}
\section{Methodology}\label{sec:3}
In this section, we will present Evolutionary Action Selection (EAS) and integrate EAS into TD3 as EAS-TD3 and enable the RL policy to learn from evolutionary actions.

\subsection{Evolutionary Action Selection (EAS)}
\label{sec:3.2}
Figure~\ref{fig:EAS} illustrates the procedure of EAS, which builds on top of PSO and TD3.
Background of both is described in Appendix~\ref{Appendix:Background}.
Taking the critic network $Q_{\mu_\theta}$ of TD3 as the fitness evaluator, EAS follows the process of PSO and evolves the action $a$ chosen by the current policy network $\mu_\theta(s)$.
The output of EAS is known as the evolutionary action $a^e$ with a higher $Q$ value than action $a$.
The pseudocode of EAS is shown in Algorithm~\ref{alg:1}.
The reasons and details of choosing the critic network $Q_{\mu_\theta}$ as the fitness evaluator are described in Appendix~\ref{sec:3.1}.

In EAS, we first add Gaussian noise to the action $a$ and make it an action set $\mathbb{A}$ containing multiple noisy actions, which serves as the initial population.
The fitness of each action is its $Q$ value generated by the critic network $Q_{\mu_\theta}$.
Secondly, we initialize the velocity vector, which determines the direction and step length when updating actions in the population.
Then, we initialize the personal best action set $\mathcal{P}=(p^b_1...p^b_n...p^b_N)$, which records the best solution of each action in the population found so far.
The best action in $\mathcal{P}$ is called the global best action $g^b$, which represents the best solution found so far.
In each generation, action $a_n^t$ in the population will be evaluated by fitness evaluator $Q_{\mu_\theta}$ to obtain the fitness $Q_n^t$.
Based on the magnitude of $Q_n^t$, we update the personal and global best action.
Moreover, the velocity is updated by Eq.~\ref{eq:velocity}, which subsequently will be used to update actions and get the next generation of population.
In Eq.~\ref{eq:velocity}, inertia weight $\omega$ describes previous velocity's influence
on current velocity.
Acceleration coefficients $c_1$ and $c_2$ represent the acceleration weights toward the personal best action and the global best action.
$r_1$ and $r_2$ are random variables uniformly distributed in $[0, 1]$.
Through several iterations, we can obtain the global best action $g^b$ with the highest $Q$ value searched so far and denote it as the evolutionary action $a^e$.
The relationship of $Q$ value between $a^e$ and $a$ is:
\begin{equation}
Q_{\mu_\theta}(s,a^e)\ge Q_{\mu_\theta}(s,a)
\label{eq:7}
\end{equation}
which reveals that EAS can increase the $Q$ value of the action so that the action will have a higher expected reward.
We claim the changing from $a$ to $a^e$ is the action evolution.
The evolutionary action $a^e$ is better than $a$ and has a higher expected reward.

\begin{algorithm}[ht]
\caption{Evolutionary Action Selection}
\label{alg:1}
\textbf{Input}:~ State $s$, action $a$, critic network $Q_{\mu_\theta}$  \\
\textbf{PSO parameters}:~Inertia weight $\omega$, acceleration coefficients $c_1,c_2$, random coefficients $r_1,r_2$\\
\textbf{Output}:~Evolutionary action $a^e$
\begin{algorithmic}[1] 
\State Extend the action $a$ with Gaussian noise $\epsilon$ to form the initial action population $\mathbb{A}=(a_1...a_n...a_N)$, $a_n=a+\epsilon_n,\epsilon\in\mathcal{N}(0,\sigma)$, $N$ is the number of actions
\State Initialize the velocity of action $\mathcal{V}=(v_1...v_n...v_N )$, $|v_n|=|a_n|=D$, $v_n\in[-v_{max}^D,v_{max}^D]$ 
\State Initialize personal best action set $\mathcal{P}=(p^b_1...p^b_n...p^b_N)$ and global best action $g^b$
\For{$t = 1$ \textbf{to} $T$}
\For{$n = 1$ \textbf{to} $N$}
\State $Q_n^t\gets Q_{\mu_\theta}(s,a_n^t)$ \Comment{Evaluate the action}
\State $p^b_n \gets a_n^t;~~\textbf{\emph{if}}~~~Q_n^t~>Q_{p_n^b}$ \Comment{Update the personal best action}
\State $ g^b \gets p^b_n;~~\textbf{\emph{if}}~~~Q_{p^b_n}>Q_{g^b} $ \Comment{Update the global best action}
\State
$ v_n^{t+1} = \omega*v_n^t +c_1*r_1*(p^b_n-a_n^t)+c_2*r_2*(g^b-a_n^t) \inlineeqnum\label{eq:velocity} $ \Comment{Update the velocity}
\State $a_n^{t+1}=a_n^{t}+v_n^{t+1} \inlineeqnum\label{eq:update_action}$ \Comment{Update the action}
\EndFor
\EndFor
\State Obtain the global best action $g^b$, representing evolutionary action $a^e$
\end{algorithmic}
\end{algorithm}

\subsection{EAS-TD3 Framework} \label{sec:3.4}
Figure~\ref{fig:3} illustrates a diagram of EAS-TD3.
See Appendix~\ref{Appendix:EAS-TD3 Pseudocode} for a pseudocode.
At each timestep $t$, the policy observes $s_t$ and outputs action $a_t$. Then, we receive a reward $r_t$ and environment transitions to next state $s_{t+1}$. These four elements make up the sample $(s_t,a_t,r_t,s_{t+1})$, which will be stored into the replay buffer $R$.
Then, EAS performs evolution on action $a$ to obtain evolutionary action $a^e$, which will be stored into an archive $\mathcal{A}$.
We draw the same mini-batches from $\mathcal{R}$ and $\mathcal{A}$ and update the current policy $\mu_\theta$ with deterministic policy gradient~\cite{DPG} and evolutionary action gradient (will be described below).
EAS adopts the way of delayed policy updates: one policy update for two $Q$ function updates, which is the same as TD3.

Through EAS, we obtain evolutionary actions, storing their corresponding state-action pairs $(s, a^e)$ into an archive $\mathcal{A}$. As mentioned in \cref{sec:3.2}, the evolutionary action has a higher expected reward than the original action. 
Consequently, we intend to make the action space of RL policy similar to the space of evolutionary actions so that the evolutionary actions can contribute to the RL policy learning.
Due to the idea given above, we construct a loss as below:
\begin{equation}
L_{evo}(\theta, \mathcal{A})=\mathbb{E}_{(s_i,a_i^e)\sim \mathcal{A}}\left[\left\|\mu_\theta(s_i)-a_i^e\right\|^2\right]
\label{eq:9}
\end{equation}
where $s_i$ and $a_i^e$ represent the state and evolutionary action sampled from $A$ respectively, and $\theta$ represents the learning parameters in RL policy $\mu_\theta$.
We call $L_{evo}$ \textit{evolutionary action gradient}.
Furthermore, we weigh the constructed loss with an extra $Q_{filter}$ as proposed in~\citet{Q-filter}:
\begin{equation}
Q_{filter}=\left\{
\begin{aligned}
1 & , ~~~if~~~Q_{\mu_\theta}(s_i,a^e_i) > Q_{\mu_\theta}(s_i, \mu_\theta(s_i)), \\
0 & , ~~~ else.
\end{aligned}
\right.
\label{eq:10}
\end{equation}

The purpose of $Q_{filter}$ is to drop out $L_{evo}(\theta, \mathcal{A})$ when the action chosen by the current RL policy is superior to the evolutionary actions in archive $\mathcal{A}$.
The reason is that if an evolutionary action generated a long time ago is sampled to update parameters, the action chosen by the current policy may be better than the previous evolutionary action.
Thus, we need to filter out those outdated evolutionary actions.
We periodically draw a batch of state-action pairs from $\mathcal{A}$ and utilize Eq.~\ref{eq:11} to update parameters.
\begin{equation}
\nabla_\theta L_{Q_{filter}evo}=Q_{filter}\nabla_\theta L_{evo}
    \label{eq:11}
\end{equation}

EAS promotes the evolution of actions, which can be used to guide strategy learning.
The introduction of $Q_{filter}$ effectively avoids a poor direction for learning when the action chosen by the current policy is superior to the previous evolutionary action in the archive.
We refer to the dual integration with TD3 as Evolutionary Action Selection-Twin Delayed Deep Deterministic Policy Gradients (EAS-TD3). 
\subsection{Theoretical Insights into EAS-TD3}
\label{sec:theoretical_insight}
Formally, we treat EAS as a virtual policy $\mu_e$, which can generate evolutionary action $a^e$ based on the action $a$ selected by the current learning policy $\mu_\theta$ (refer to TD3's policy).
Given that we have Proposition~\ref{proposition:1} (see Appendix~\ref{Appendix:Proof} for proof), where $Q_{\mu_e}$ and $Q_{\mu_\theta}$
are the state-action value function of corresponding policies, $p$ is the transition probability.
It indicates that $\mathbb{E}_{a^e\sim \mu_e(s)}[Q_{\mu_e}(s,a^e)]$ is greater than $\mathbb{E}_{a\sim \mu_\theta(s)}[Q_{\mu_\theta}(s,a)]$.
That is, $\mu_e$ outperforms $\mu_\theta$. Evolutionary actions generated by EAS have higher expected rewards.
We employ an archive to store evolutionary actions for reuse.
With the guidance of evolutionary action gradient (Eq.~\ref{eq:9}), evolutionary actions drive the current learning policy towards the state-action space with higher expected rewards as illustrated in Figure~\ref{fig:6} (see Section~\ref{subsec:evo_act_eval} for detailed explanation).
With the $Q$-filter, even in the worst case, all evolutionary actions sampled from the archive are filtered and Eq.~\ref{eq:9} drops to zero. It will \textbf{NOT} hinder the policy learning.
In summary, compared with TD3, we additionally introduce evolutionary action to guide the policy learning towards the state-action space with higher expected rewards.

\begin{proposition}
\label{proposition:1}
EAS policy $\mu_e$ optimizes the action $a$ (generated by $\mu_\theta$) in the direction of increasing $Q_{\mu_\theta}$ (estimated by the critic) to obtain $a^e$.
Hence, for arbitrary state $s$, there exists inequality $\xi$: $\mathbb{E}_{a^e \sim\mu_e(s)}[Q_{\mu_\theta}  (s,a^e)] \ge \mathbb{E}_{a\sim\mu_\theta(s)}[Q_{\mu_\theta} (s,a)]$. Then, we hold that:
\begin{equation*}
\mathbb{E}_{a^e\sim\mu_e(s)}[Q_{\mu_e}(s,a^e)]\ge\mathbb{E}_{a\sim\mu_\theta(s)}[Q_{\mu_\theta}(s,a)]
\label{eq:appendix_proposition_1}
\end{equation*}
\end{proposition}

\begin{figure}[htb]
    \centering
    \includegraphics[width=0.70\columnwidth]{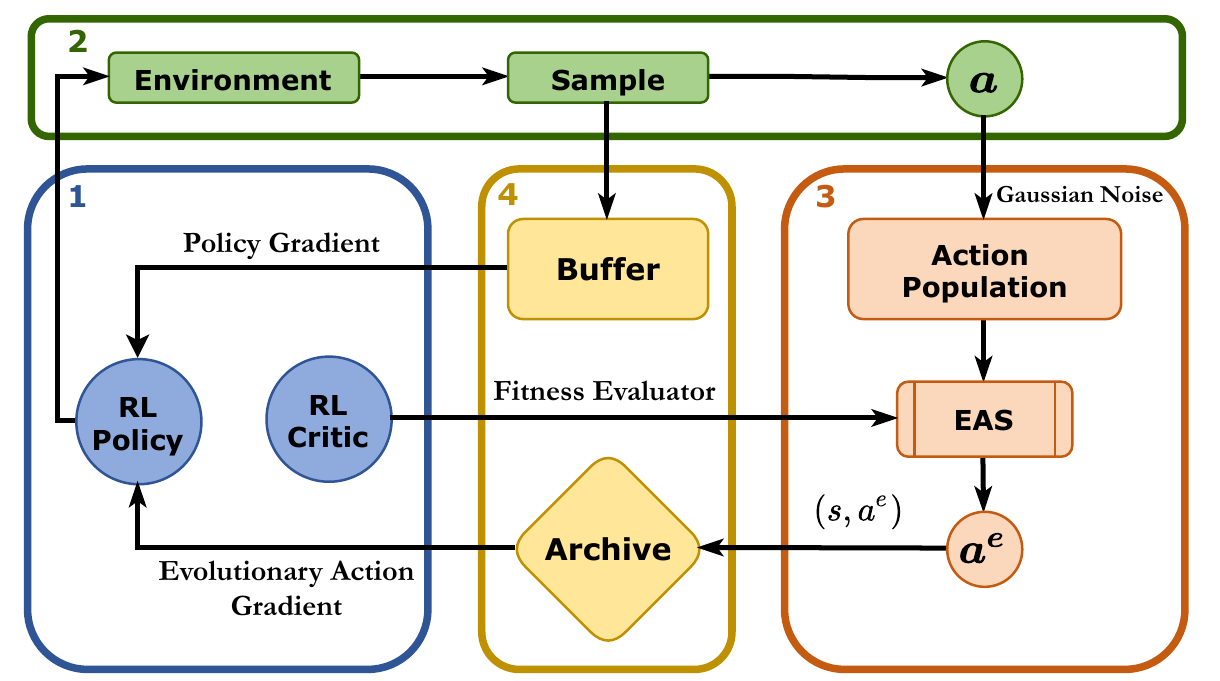}
    \caption{A high-level view of EAS-TD3.
    RL policy interacts with the environment to generate the sample $(s, a, r, s')$, which will be stored in a replay buffer.
    Action $a$ is added with Gaussian noise to form an evolutionary population.
    Then use EAS to update the action population to obtain the evolutionary action $a^e$. The state-action pair $(s,a^e)$ will be stored in an archive.
    We draw batches from the replay buffer and archive to update the current RL policy with policy gradient and evolutionary action gradient.}
    \label{fig:3}
\end{figure}
\section{Experiment}
\label{sec:experiment}

The main purpose of this section is to investigate the mechanism of EAS and the performance of EAS-TD3 compared to other evolutionary reinforcement learning methods.
Firstly, we construct a toy example to confirm our motivation, that the ERL methods utilize EA to optimize high-dimensional parameter space, which will lead to performance degradation.
Then, we conduct extensive experiments on the continuous locomotion tasks from MuJoCo and analyze why EAS-TD3 works better than ERL methods and TD3.
Moreover, we perform ablation studies to analyze the effect of components and hyperparameters.
Finally, we carry out a visual demonstration, which reveals the mechanism of EAS and how evolutionary action promotes strategy learning.

\subsection{A Toy Example}
\label{sec:toy_example}

As mentioned in~\cref{sec:intro}, ERL methods generally choose the high-dimensional parameter space as the target of evolution, which is tricky for EA to optimize and may lead to the collapse performance of the overall approach.
To confirm this, we set up a toy example.
Specifically, we increase the hidden layer of the policy network and carry out experiments on the Inverted Double Pendulum from OpenAI gym.
As a fairly easy continuous control task for most modern algorithms, Inverted Double Pendulum has an 11-dimensional state space and 1-dimensional action space.

Figure~\ref{fig:toy_example} shows the comparative performance of EAS-TD3, CEM-TD3, TD3.
As shown in ~\cref{fig:toy_c}, the performance of CEM-TD3 decays with the increase of policy parameters.
As the increase of hidden layer, the dimension of parameter space needed to be optimized increases dramatically.
This makes it difficult for CEM to search for decent parameters of the policy network, which attenuates the performance of policy population and the contribution of the evolutionary part.
Subsequently, the stagnant policy population produces worthless samples to store in the replay buffer.
The RL part may be drowning in these useless samples.
In short, taking the policy network parameter space as the evolutionary target may lead to performance degradation.
Correspondingly, the performance of EAS-TD3 is not affected as shown in~\cref{fig:toy_a}, since its evolutionary part optimizes the low-dimensional action space, which is rarely changed for a given task.

\subsection{Experimental Setup}
We select four continuous control locomotion tasks from OpenAI Gym~\cite{gym} simulated by MuJoCo~\cite{Mujoco}: HalfCheetah-v3, Walker2d-v3, Ant-v3, Humanoid-v3.
Besides, we modified these four tasks to give a delayed cumulative reward only after every $f_{reward}$ step (or when the episode terminates).
$f_{reward}$ is chosen to 20 for Walker2d and HalfCheeth and 10 for Ant and Humanoid.
We present the average reward and the associated standard deviation over 10 runs.
For each run, we test the learning policy on 10 evaluation episodes every 5000 steps.
In all figures of learning curves, unless specified otherwise, the x-axis represents the number of steps performed in the environment and the y-axis represents the mean return obtained by the policy.

All hyperparameters for TD3 are the same as those in the original paper~\cite{TD3} by default.
Here, we only provide the unique hyper-parameters of the EAS-TD3.
With regard to PSO, inertia weight $\omega$ is 1.2, acceleration coefficients $c_1$,$c_2$ are both 1.5, the number of iterations $T$ is 10, $v_{max}$ is 0.1, random coefficients $r_1$,$r_2$ are random numbers from 0 to 1.
The size of archive $\mathcal{A}$ is 100,000 for all environments.
the action population size is 10.
More details about parameters can be found in Appendix~\ref{Appendix:Hyperparameters}, which contains all hyperparameter settings and descriptions.

Our method is compared against the official implementations for TD3~\cite{TD3_Code}, CEM-TD3 ~\cite{CEM-RL}, CERL~\cite{CERL}, PDERL~\cite{PDERL}, ERL~\cite{ERL}.
These baselines contain a series of studies on the combination of EA and RL.
The population size of all baselines is set to 10.
\begin{figure*}[htb]
     \vspace{10pt}
     \centering
     \begin{subfigure}[t]{0.25\textwidth}
         \centering
         \includegraphics[width=\textwidth]{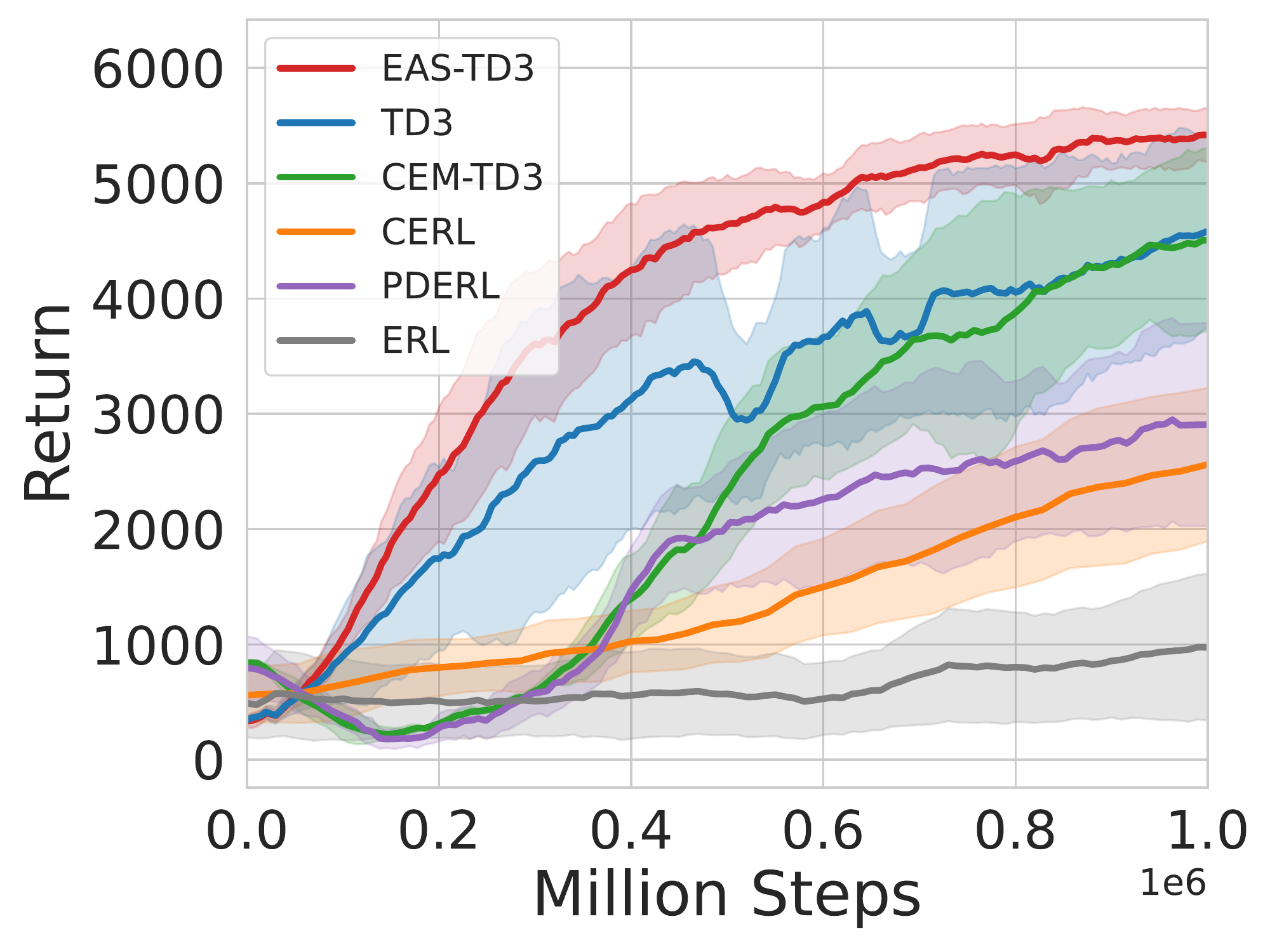}
         \caption{Ant-v3}
         \label{fig:4a}
     \end{subfigure}%
     \hfill
     \begin{subfigure}[t]{0.25\textwidth}
         \centering
         \includegraphics[width=\textwidth]{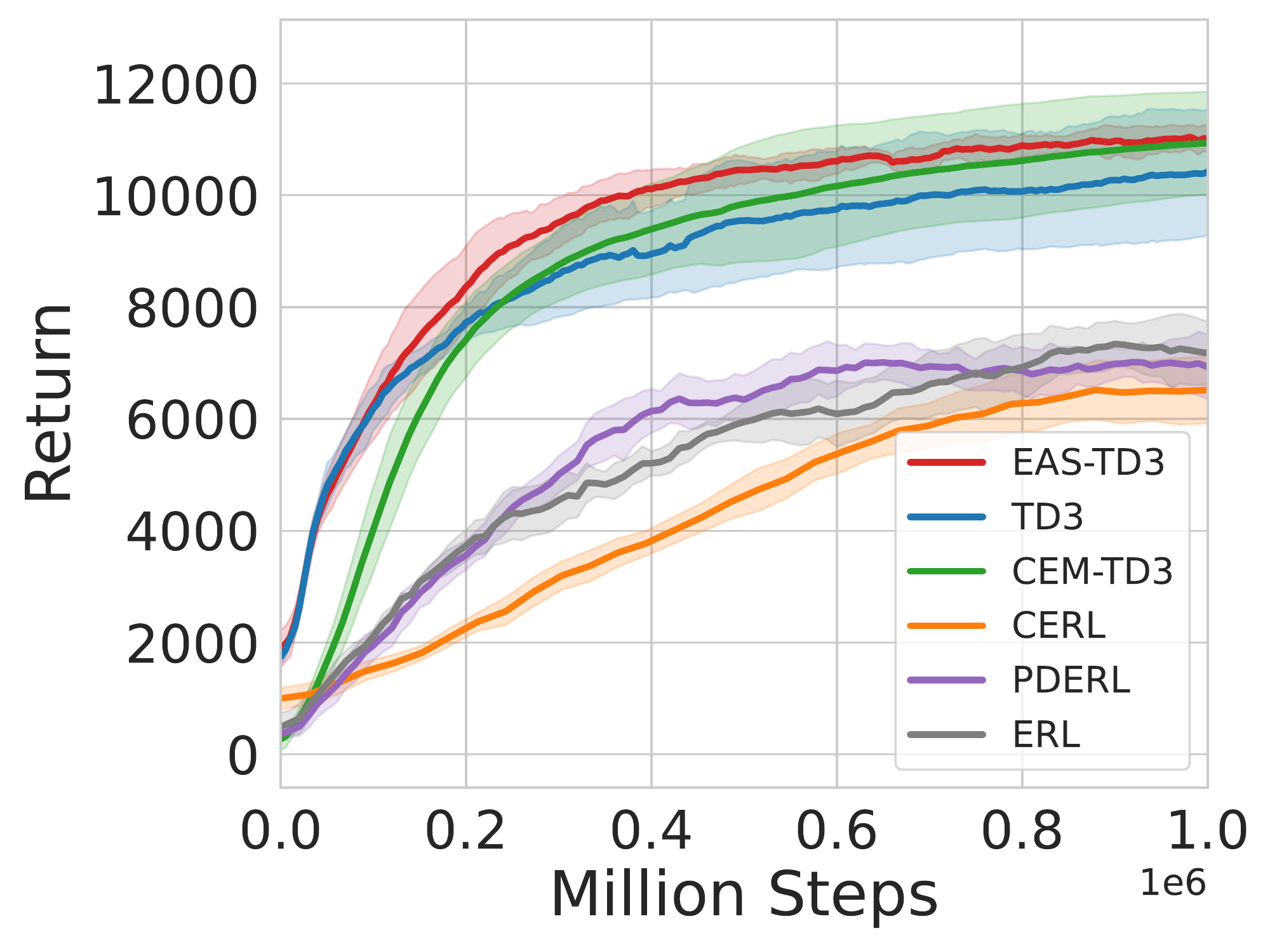}
         \caption{HalfCheetah-v3}
         \label{fig:4b}
     \end{subfigure}%
     \hfill
     \begin{subfigure}[t]{0.25\textwidth}
         \centering
         \includegraphics[width=\textwidth]{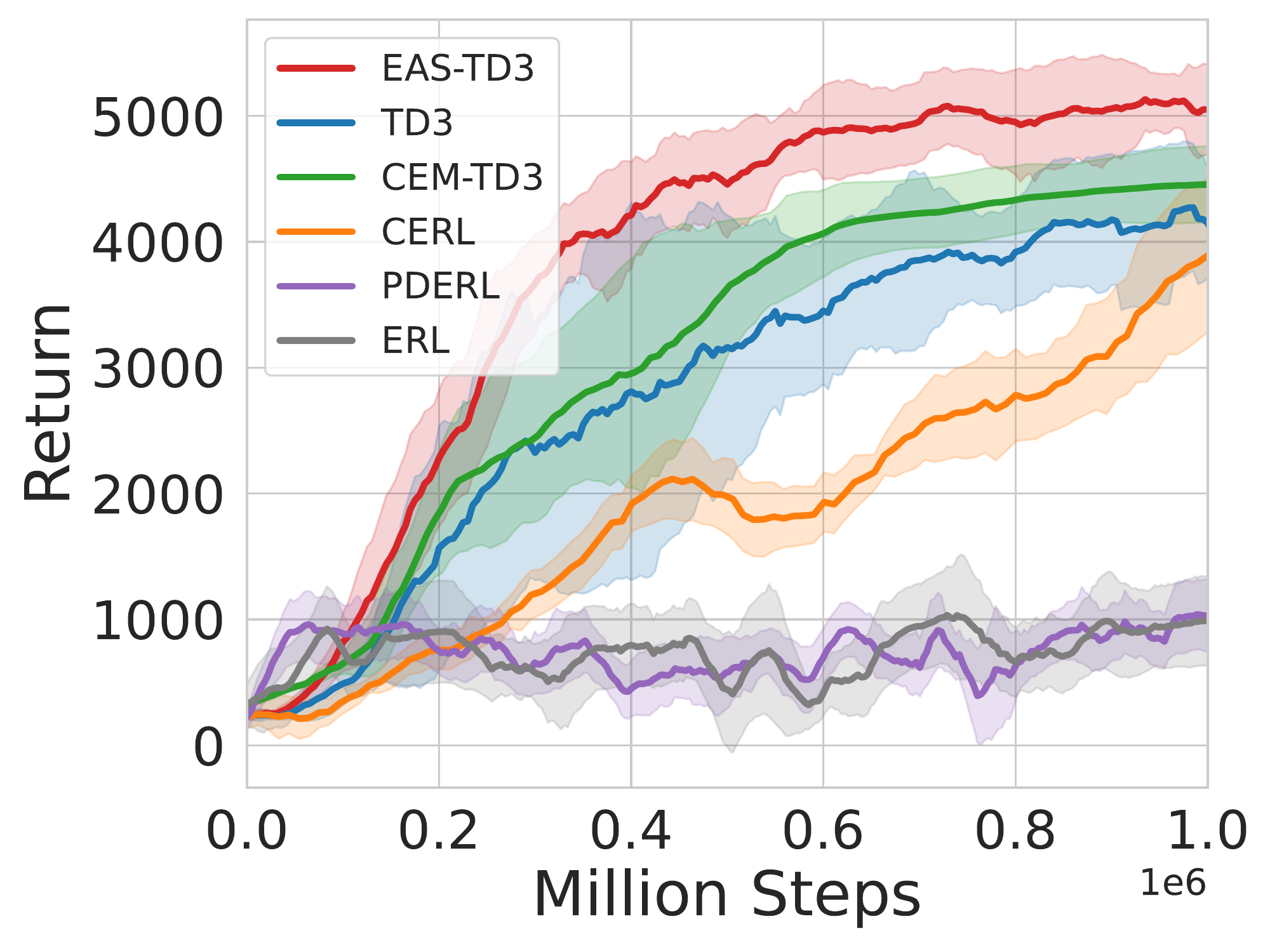}
         \caption{Walker2d-v3}
         \label{fig:4c}
     \end{subfigure}%
     \hfill
     \begin{subfigure}[t]{0.25\textwidth}
         \centering
         \includegraphics[width=\textwidth]{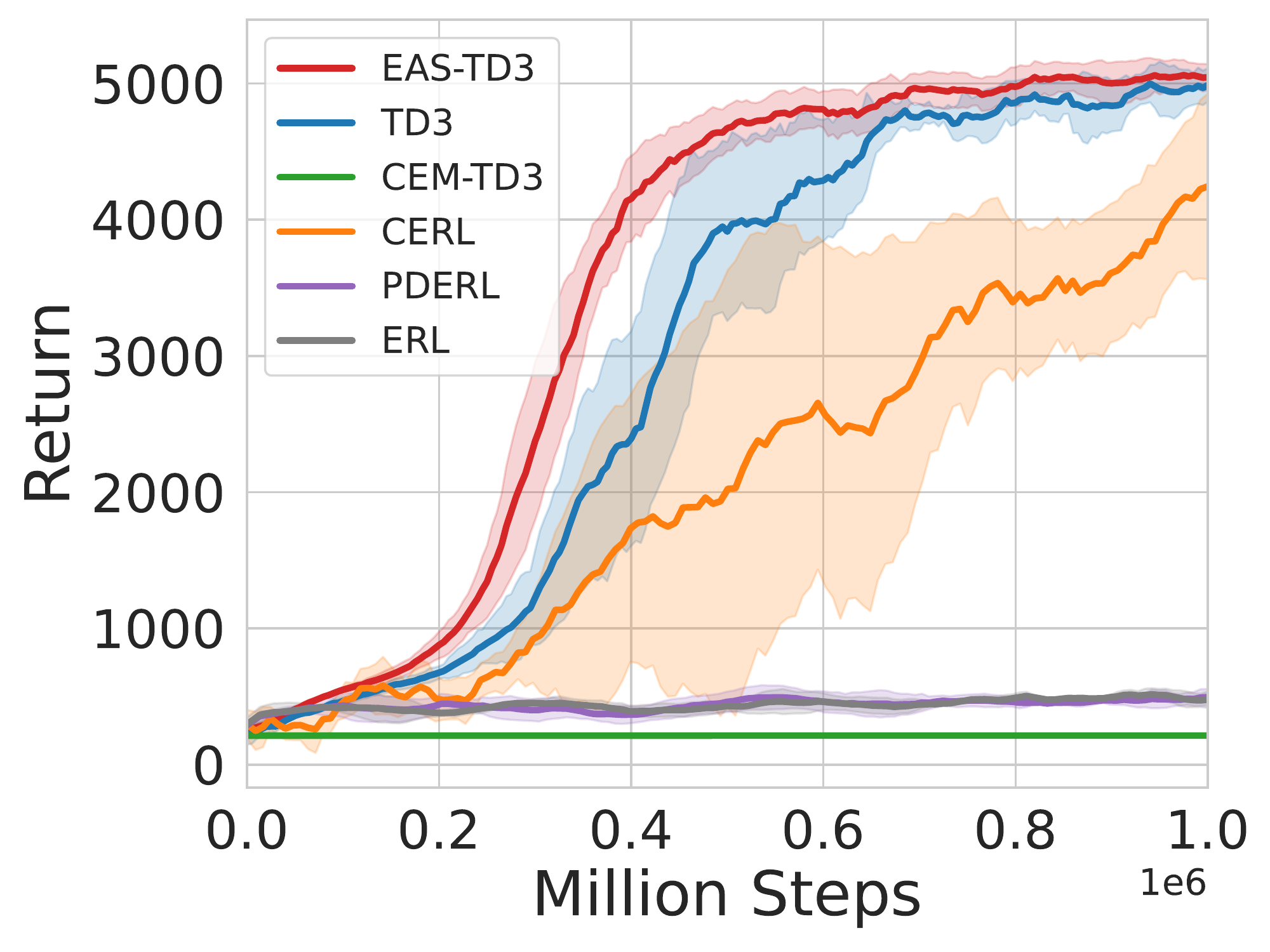}
         \caption{Humanoid-v3}
         \label{fig:4d}
     \end{subfigure}%
     
     \bigskip
     \vspace{10pt}
     
     \begin{subfigure}[ht]{0.25\textwidth}
         \centering
         \includegraphics[width=\textwidth]{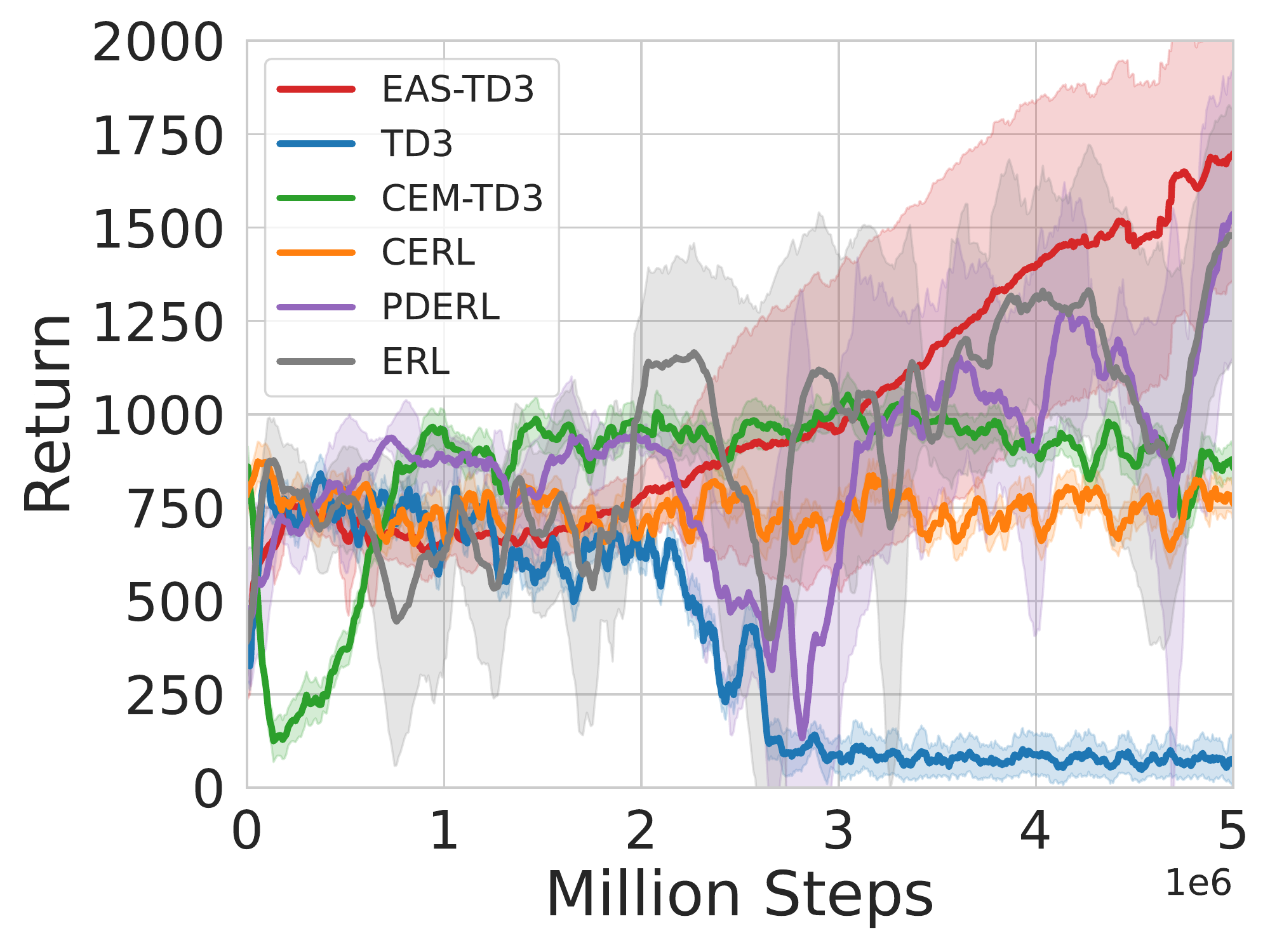}
         \caption{DelayedAnt-v3}
         \label{fig:4e}
     \end{subfigure}%
     \hfill
     \begin{subfigure}[ht]{0.25\textwidth}
         \centering
         \includegraphics[width=\textwidth]{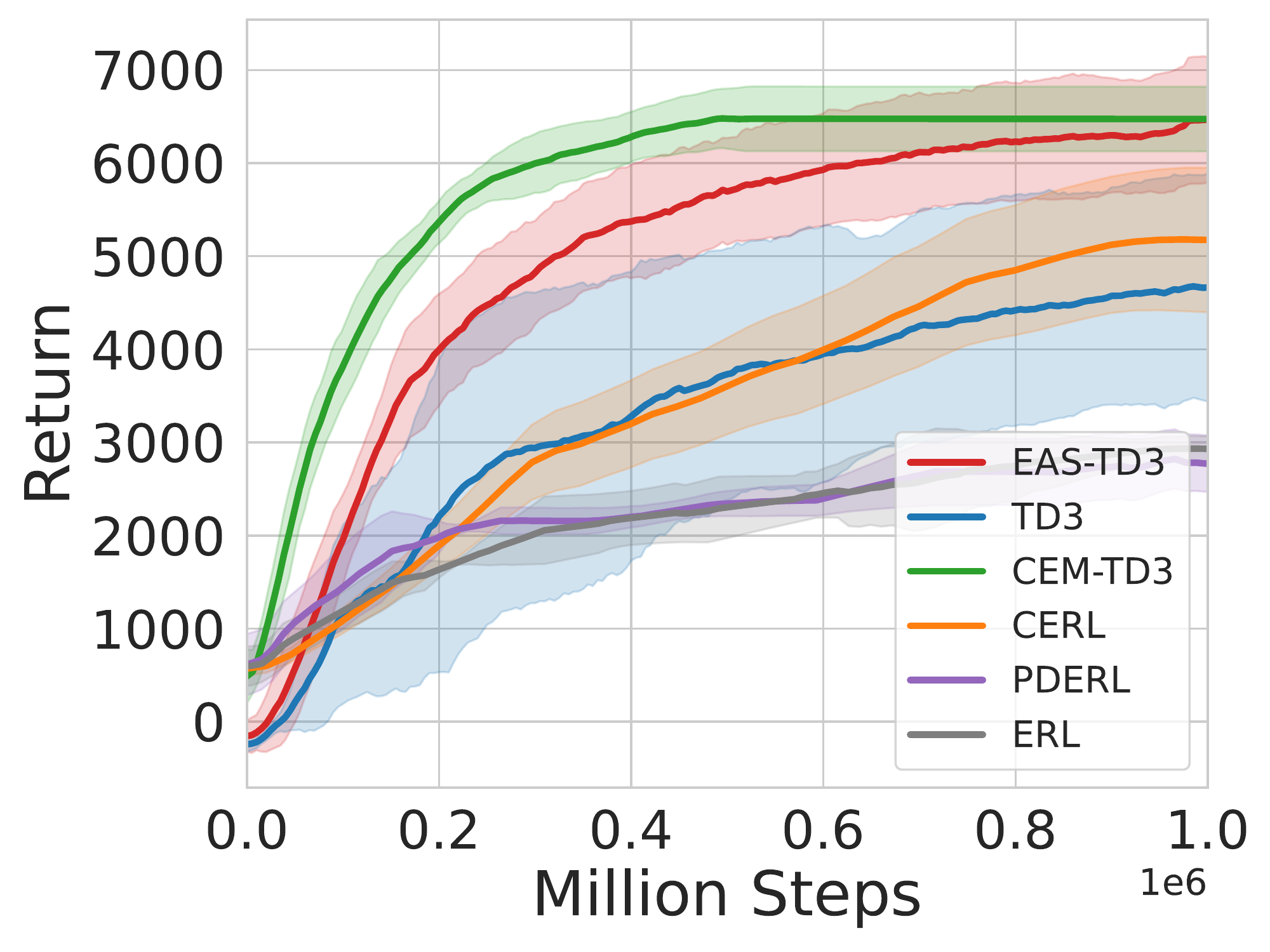}
         \caption{DelayedHalfCheetah-v3}
         \label{fig:4f}
     \end{subfigure}%
     \hfill
     \begin{subfigure}[ht]{0.25\textwidth}
         \centering
         \includegraphics[width=\textwidth]{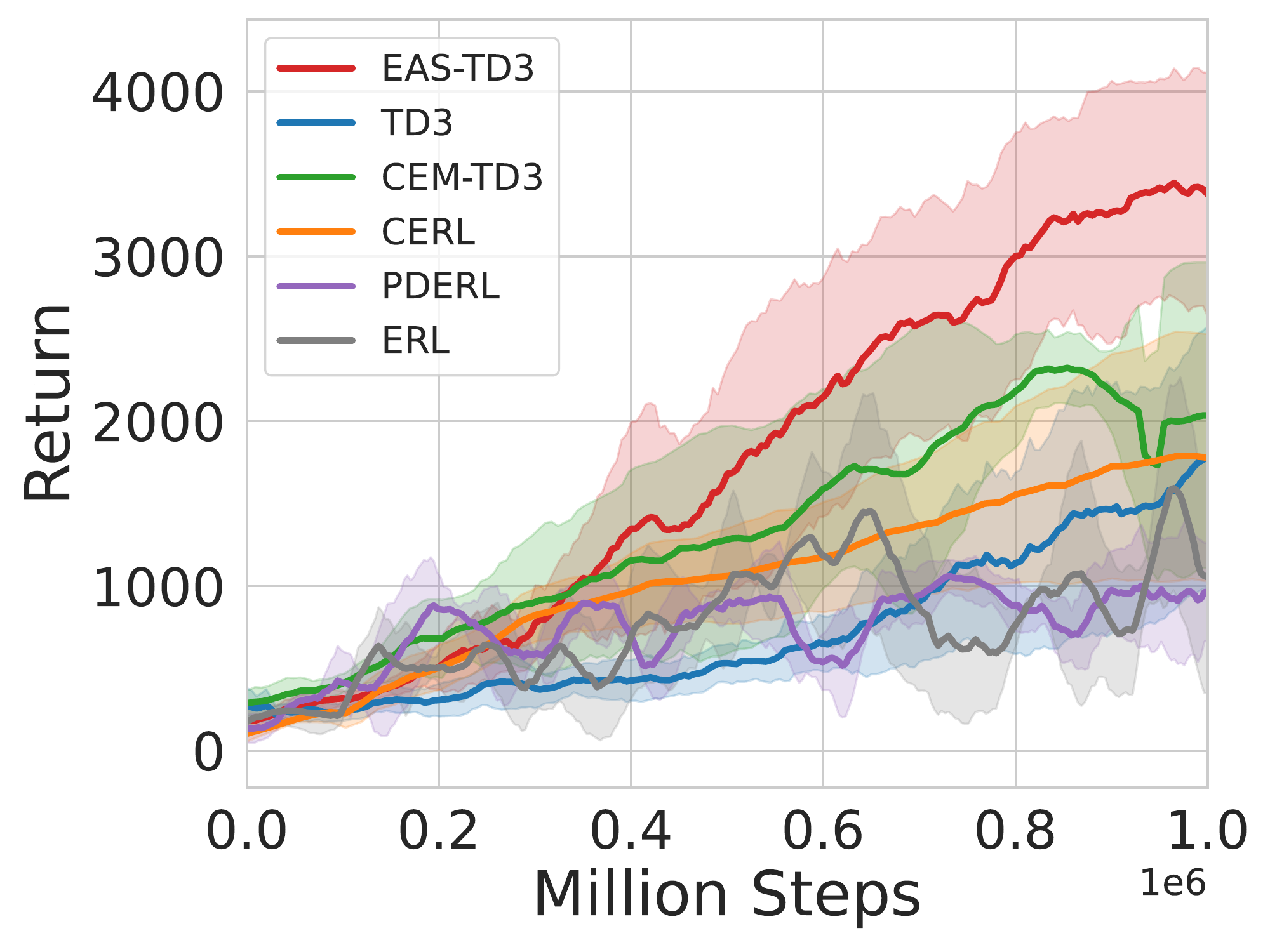}
         \caption{DelayedWalker2d-v3}
         \label{fig:4g}
     \end{subfigure}%
     \hfill
     \begin{subfigure}[ht]{0.25\textwidth}
         \centering
         \includegraphics[width=\textwidth]{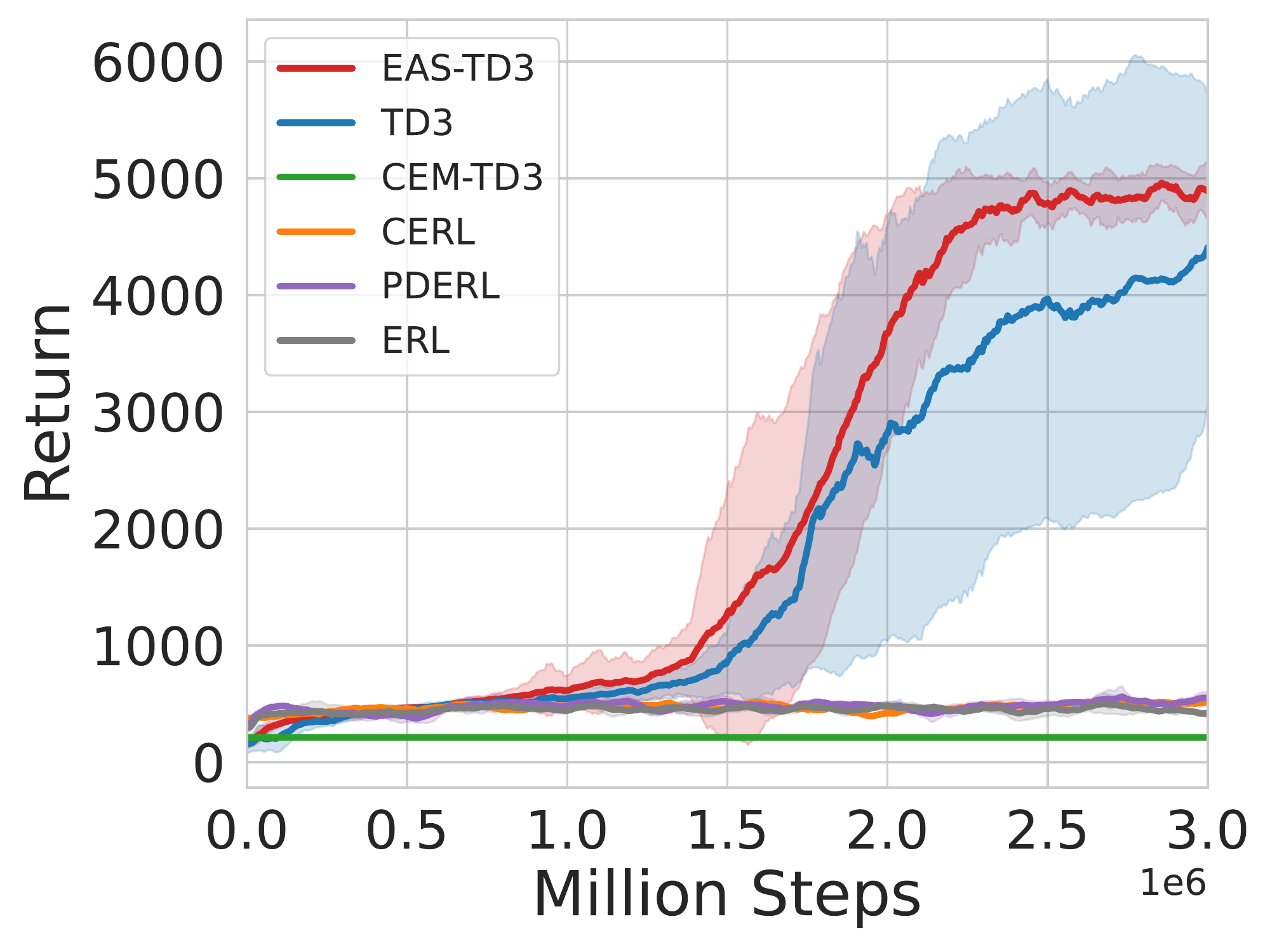}
         \caption{DelayedHumanoid-v3}
         \label{fig:4h}
     \end{subfigure}%
    \caption{The learning curves in MuJoCo environments (up) and DelayedMuJoCo environments (down). The shaded area represents mean ± standard deviation over the
10 runs.}
    \label{fig:4}
\end{figure*}
\subsection{Experiments on MuJoCo and Delayed MuJoCo Environments}

Figure~\ref{fig:4} demonstrates the reward curves in MuJoCo and Delayed MuJoCo environments.
In general, EAS-TD3 performs consistently well across all MuJoCo environments and most Delayed MuJoCo environments except for DelayedHalfCheetah-v3, indicating that EAS plays a significant role in the process of policy learning.

\textbf{Compared to ERL methods}.
EAS-TD3 performs a large improvement on environments with high-dimensional state and action space like Humanoid, Ant, Walker2d.
Why are ERL methods not performing as well as EAS-TD3 on these tasks?
Note that the parameters of the policy network are associated with the state and action dimensions of the environment.
For environments with higher dimensions, their policy networks also have more parameters, and more timesteps are needed for EA to search for good policies.
For example, a policy network in Humanoid-v3 task (state space 376 dims \& action space 17 dims) consisting of a hidden layer with 400 and 300 nodes will have more than 270,000 parameters, which takes more timesteps to learn a decent policy for gradient-free evolutionary methods.
Therefore, in high-dimensional tasks like Humanoid and Ant, the contribution of the evolutionary part in ERL methods will be weakened, resulting in the reduction of learning efficiency or even complete failure, which is consistent with the performance shown in the toy example.
However, EAS-TD3 shifts the target of evolution from high-dimensional parameter space to low-dimensional action space, so as to avoid the disaster of parameter dimension growth.
The evolutionary part in EAS-TD3 focuses on evolving better actions and promoting policy learning through evolutionary action gradients.
It is the main reason why EAS-TD3 performs better.
Besides, CEM-TD3 performs better on HalfCheetah and Delayed HalfCheetah, probably because CEM can search decent policies on these simple tasks.
Moreover, we remark that the evolutionary part may even hinder the early learning of RL part in ERL methods.
It is reflected in Ant and Walker2d tasks, where CEM-TD3, the best performer among the ERL methods, has a slightly slower learning efficiency than TD3.
The reason may be attributed to the lackluster performance of the evolution population, which brings a huge amount of worthless experiences to the replay buffer, thus hindering the learning process of RL part.
EAS-TD3 has no such concern, since the evolution part does not directly train the policy network.

\textbf{Compared to TD3}.
EAS-TD3 outperforms TD3  more or less in all environments.
Compared with TD3, the increment of EAS-TD3 only lies in the influence of the evolutionary action gradient, which indicates the performance enhancement does come from the introduction of EAS.
The evolutionary action can promote policy learning within different tasks.
In HalfCheetah, EAS-TD3 achieves a minor improvement in experiments with multiple random seeds.
Since TD3 has already shown promising results in this environment, the evolutionary action can slightly enhance the sampling efficiency and performance.
At the same time, EAS-TD3 performs well in high-dimensional difficult tasks such as Ant, Walker2d, and Humanoid, as the gradient of evolutionary action rapidly drives the policy's action space toward regions with higher expected reward and greatly improves learning speed and final performance.
We will discuss more how EAS works in the following section.
\begin{figure*}[htb]
    \centering
    \includegraphics[width=0.98\linewidth]{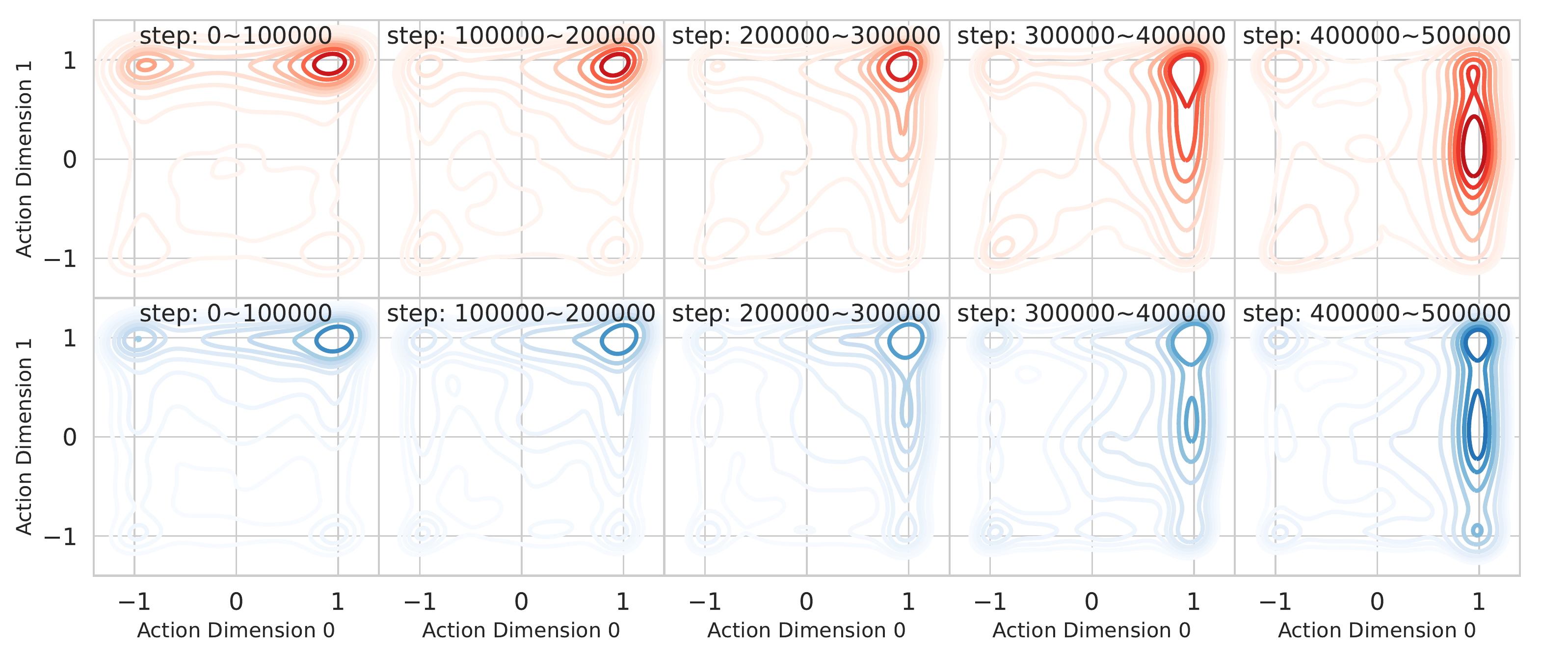}
    \caption{
    The distribution of the first two action dimensions during the training process.
    The red contour represents the current policy’s actions and the blue contour represents the corresponding evolutionary actions.
    The performance of the policy will gradually improve over timesteps.
    The action distribution plot of each column indicates a higher expected reward than the previous column.
    Evolutionary actions predict the action space with a higher expected reward in advance and guide the current policy to move towards there.}
    \label{fig:6}
\end{figure*}

\subsection{Ablation Studies}
\label{sec:4.3}
\textbf{The size of archive $\mathcal{A}$}.
The capacity of archive $\mathcal{A}$ is generally not large, in that the policy will continue to learn while evolutionary action selection is based on the policy at that time.
Therefore, we should control the size of $\mathcal{A}$ to keep evolutionary actions fresh.
Otherwise, the outdated evolutionary actions will not play a role in guiding policy learning.
We select different archive sizes (10,000, 50,000, 100,000, 500,000) and perform an ablation study on the size of $\mathcal{A}$.
As shown in Figure~\ref{fig:ablation_ab}, the archive size affects the performance.
$\mathcal{A}$ stores state-action pairs $(s,a^e)$ corresponding to evolutionary actions.
As discussed earlier, since the evolutionary actions need to keep fresh, the archive size shouldn’t be too large.
Otherwise, the archive $\mathcal{A}$ will store outdated actions, which may weaken the role of evolutionary action.
In contrast, it shouldn’t be too small, or the diversity of actions in $\mathcal{A}$ may be poor. This will lead to the frequent sampling of limited evolutionary actions and hinder the exploration ability of policy. From Figure~\ref{fig:ablation_ab}, 100,000 seems to be a reasonable size for these tasks.

\textbf{The effect of $Q$ filter}.
$Q$ filter is employed to drop out evolutionary action gradient when the action chosen by the current policy is better than the evolutionary action in the archive as discussed in Section~\ref{sec:3.4}.
Since RL policy is constantly learning and improving, evolutionary actions in the archive may be inferior to actions chosen by the current policy.
Figure~\ref{fig:ablation_cd} show the influence of $Q$ filter.
Without $Q$ filter, the learning process of HalfCheetah will be affected by the obsolescence of evolutionary actions.
We also found that in a high-dimensional environment like Ant, even without $Q$ filter, it can also achieve good performance. 
It's may be due to the reason that the appropriate size of the archive plays a role.
The evolutionary actions in the archive are kept fresh and well used to guide strategy learning.
\begin{figure}[htb]
     \centering
     \begin{minipage}[t]{0.48\textwidth}
     \begin{subfigure}[t]{0.50\columnwidth}
         \centering
         \includegraphics[width=\columnwidth]{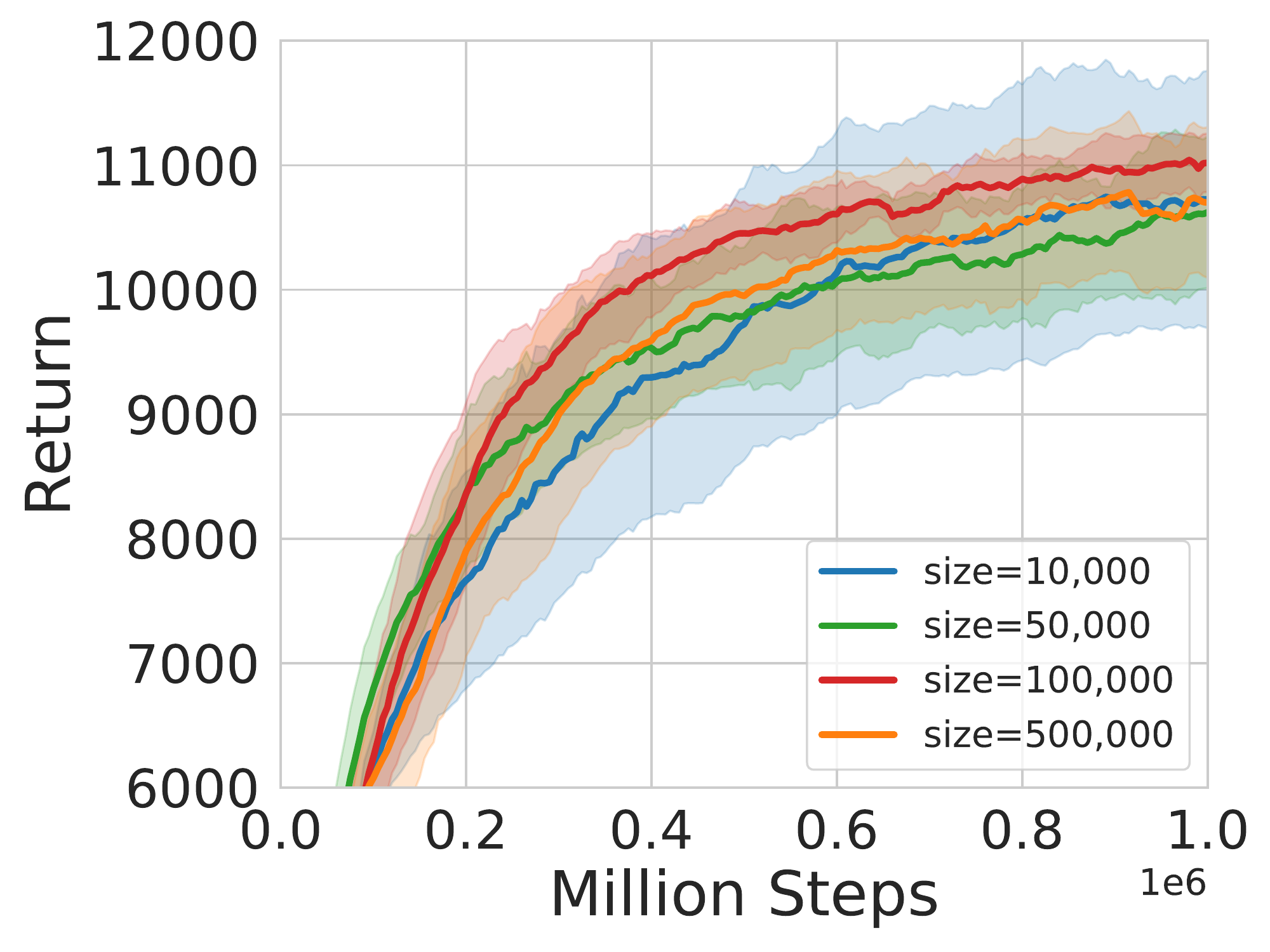}
         \caption{HalfCheetah-v3}
         \label{fig:5a}
     \end{subfigure}%
     \begin{subfigure}[t]{0.50\columnwidth}
         \centering
         \includegraphics[width=\columnwidth]{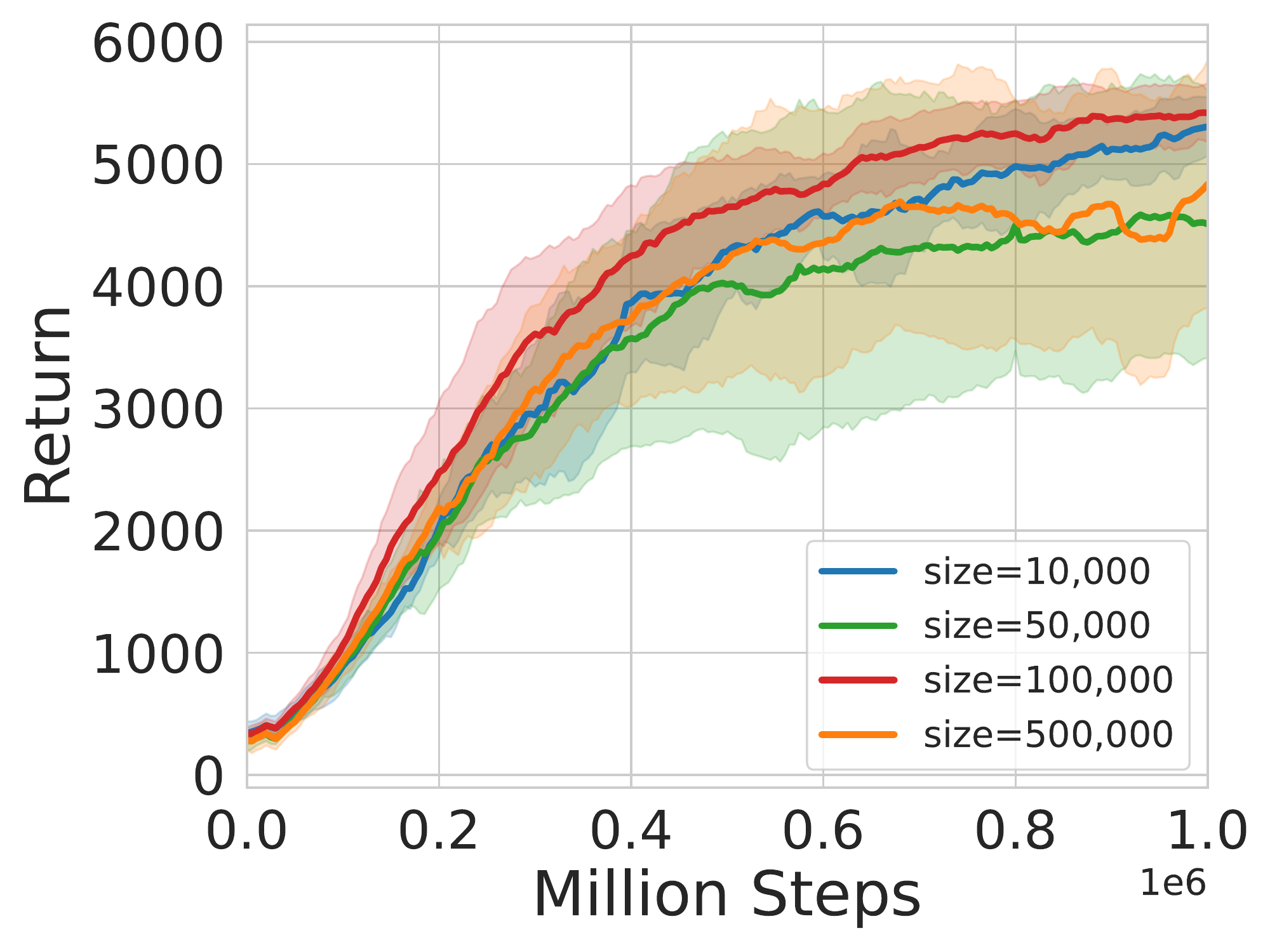}
         \caption{Ant-v3}
         \label{fig:5b}
     \end{subfigure}%
\caption{Ablation study (mean $\pm$ standard deviation) on the size of archive $\mathcal{A}$.
100,000 seems to be a reasonable size.}
\label{fig:ablation_ab}
\end{minipage}
\hfill
\begin{minipage}[t]{0.48\textwidth}
\begin{subfigure}[t]{0.50\columnwidth}
    \centering
    \includegraphics[width=\columnwidth]{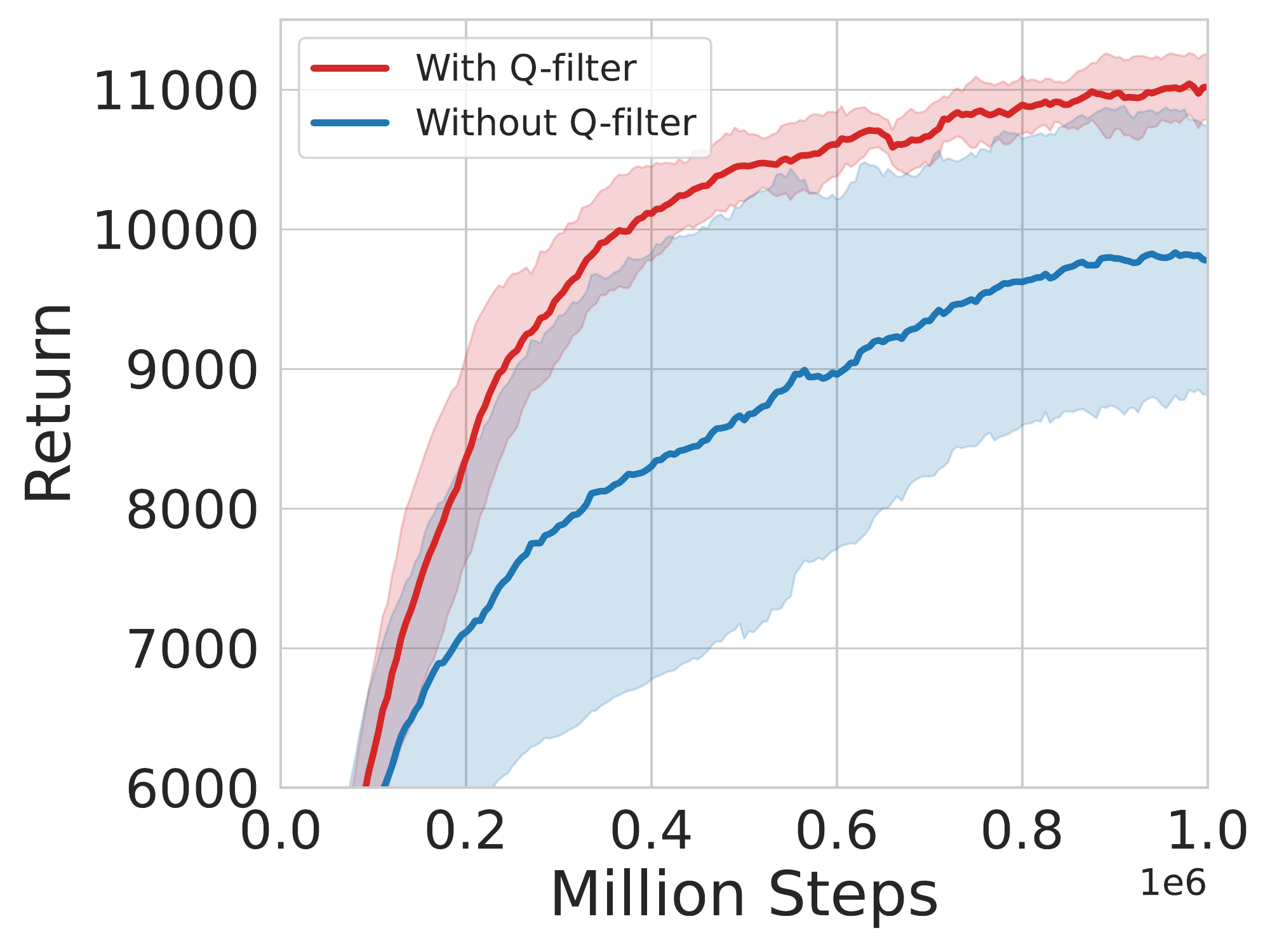}
    \caption{HalfCheetah-v3}
    \label{fig:5c}
\end{subfigure}%
\begin{subfigure}[t]{0.50\columnwidth}
    \centering
    \includegraphics[width=\columnwidth]{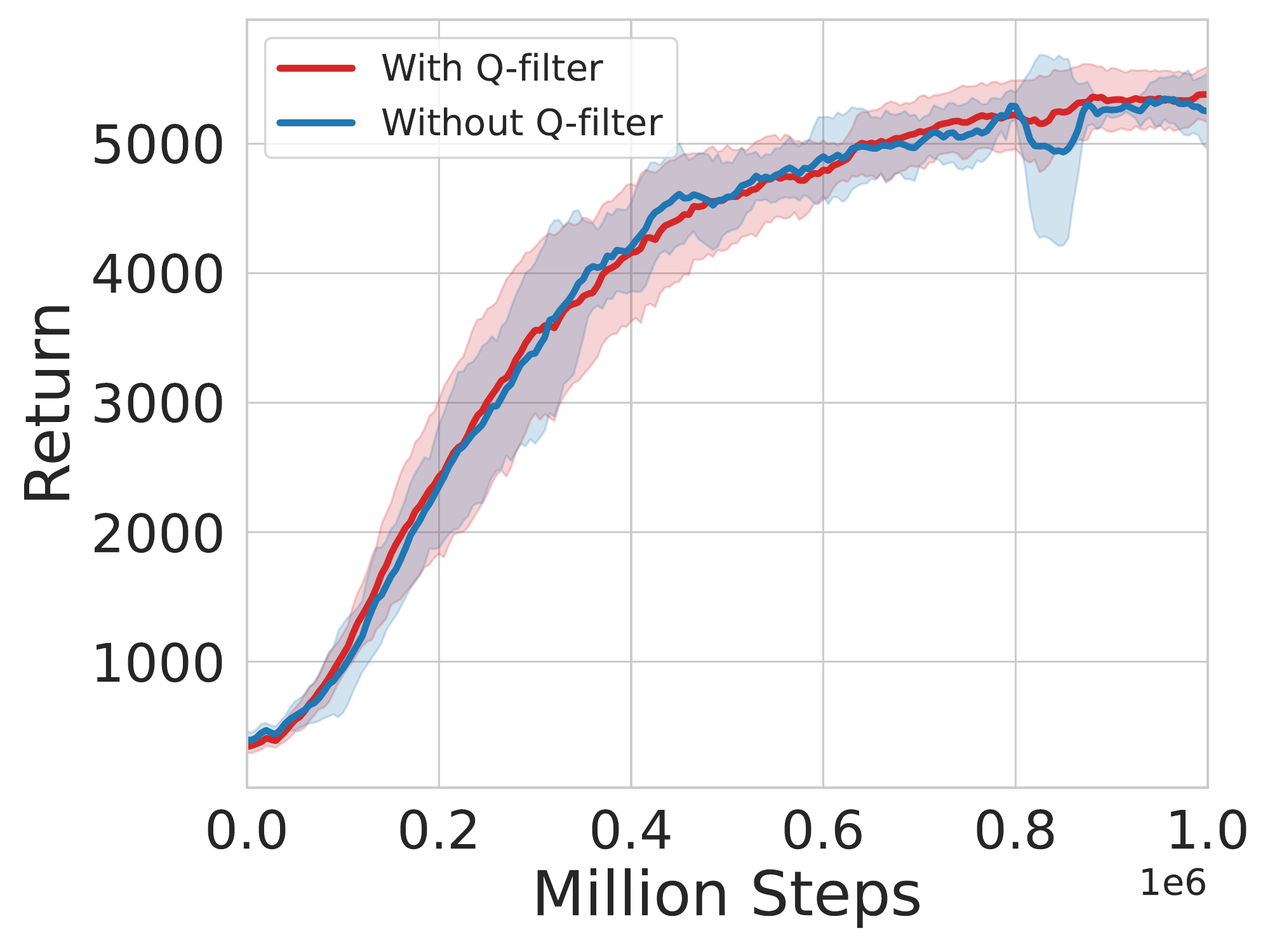}
    \caption{Ant-v3}
    \label{fig:5d}
\end{subfigure}%
\caption{Ablation study (mean $\pm$ standard deviation) on the effect of $Q$ filter.}
\vskip -0.1in
\label{fig:ablation_cd}
\end{minipage}
\label{fig:5}
\end{figure}

\subsection{Evolutionary Action Evaluation}
\label{subsec:evo_act_eval}
In this section, we will investigate the mechanism of EAS and mainly answer a question: how do evolutionary actions promote policy learning?

Figure~\ref{fig:6} shows the first two-dimensional distributions of actions chosen by the policy and the corresponding evolutionary actions during the training process of Walker2d-v3.
Each plot demonstrates the distribution of all actions in every 100,000 timesteps.
The performance of the policy will gradually improve over timesteps.
Therefore, the action distribution plot of each column in Figure~\ref{fig:6} indicates a higher expected reward than the previous column.
From Figure~\ref{fig:6}, we can see that evolutionary actions predict the action space with higher expected reward in advance and guide the current policy to move towards there, which may explain why EAS works.
For example, within the 0 to 100,000 timesteps (corresponding to the two plots in the first column), the distribution of the evolutionary action exhibits a trend from the upper right downward.
The downward trend subsequently becomes more pronounced.
Correspondingly, the policy action distribution shows the same trend under the guidance of both the policy gradient and evolutionary action gradient.
This indicates that the evolutionary action gradient promotes strategy learning in the same direction as the policy gradient, which significantly accelerates the progress of policy learning.
From Figure~\ref{fig:6}, we also discover that the space of evolutionary actions doesn't deviate too much from the current policy's action space, which is served as a gentle mentor and guides the current policy step by step.
See Appendix~\ref{Appendix:More Evolutionary Action Evaluation} for a complete plot and more analysis.
In conclusion, with the aid of evolutionary action gradient, we can make good use of evolutionary actions and finally play a role in the learning process.
\begin{wrapfigure}{r}{0.40\textwidth}
    \centering
    \includegraphics[width=0.97\linewidth]{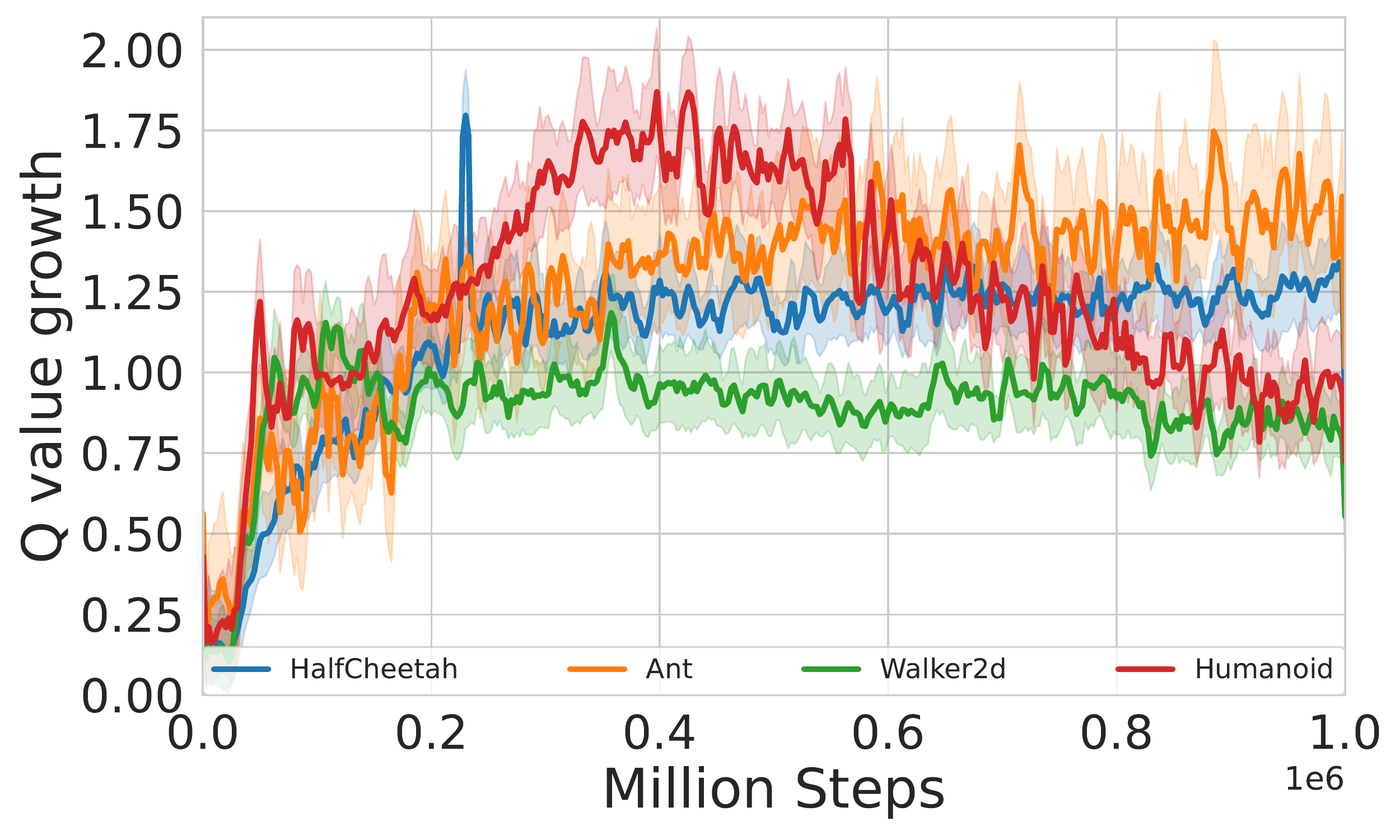}
    \caption{Increase of $Q$ value by EAS. In the training process, the growth of $Q$ value by EAS increases gradually and tends to be stable. }
    \label{fig:7}
    \vspace{-12pt}
\end{wrapfigure}

In addition, we recorded the $Q$ value of each action and the corresponding evolutionary action throughout the training process in MuJoCo environments.
Figure~\ref{fig:7} demonstrates the magnitude of the $Q$ value growth before and after the action evolution.
In the training process, the growth of $Q$ value increases gradually and tends to be stable, which indicates EAS can provide long-term and stable guidance for the whole learning process.
Besides, in the period when the learning curve rises rapidly (200,000 to 600,000 timesteps), the growth of $Q$ value is also large.
Especially in the Humanoid environment, the learning curve rises fastest during 300,000 to 400,000 timesteps while the $Q$ value growth is also largest, which indicates that the performance improvement does come from EAS.
\section{Conclusion and Further Work}
\label{sec:conclusion}

This paper proposes Evolutionary Action Selection (EAS) which apply evolution for the action chosen by the policy.
We integrate EAS into TD3 as EAS-TD3, which is a novel hybrid method of EA and RL.
Compared with other hybrid methods, our approach transforms the target of evolution from the high-dimensional parameter space to the low-dimensional action space.
With the promotion of evolutionary actions, we train policies that substantially outperform those trained directly by TD3 and other compared methods.
Extensive experiments and empirical analyses demonstrate that evolving action is a favorable alternative to evolving the parameters of the policy network.
We believe that our work can trigger follow-up studies to address several interesting open questions.
In future work, we can integrate the idea of Quality-Diversity (QD)~\cite{QDNature, QD}, which not only evolve high-quality actions but also evolve diverse actions to deal with the environment with deceptive rewards.

\bibliography{neurips_2022.bib}

\begin{thebibliography}{50}
\providecommand{\natexlab}[1]{#1}
\providecommand{\url}[1]{\texttt{#1}}
\expandafter\ifx\csname urlstyle\endcsname\relax
  \providecommand{\doi}[1]{doi: #1}\else
  \providecommand{\doi}{doi: \begingroup \urlstyle{rm}\Url}\fi

\bibitem[Ackley(1992)]{EARL1992}
David Ackley.
\newblock Interactions between learning and evolution.
\newblock \emph{Artificial life II}, 1992.

\bibitem[Bodnar et~al.(2020)Bodnar, Day, and Li{\'o}]{PDERL}
Cristian Bodnar, Ben Day, and Pietro Li{\'o}.
\newblock Proximal distilled evolutionary reinforcement learning.
\newblock In \emph{Proceedings of the AAAI Conference on Artificial
  Intelligence}, pages 3283--3290, 2020.

\bibitem[Brockman et~al.(2016)Brockman, Cheung, Pettersson, Schneider,
  Schulman, Tang, and Zaremba]{gym}
Greg Brockman, Vicki Cheung, Ludwig Pettersson, Jonas Schneider, John Schulman,
  Jie Tang, and Wojciech Zaremba.
\newblock Openai gym.
\newblock \emph{arXiv preprint arXiv:1606.01540}, 2016.

\bibitem[Casas(2017)]{DDPG}
Noe Casas.
\newblock Deep deterministic policy gradient for urban traffic light control.
\newblock \emph{arXiv preprint arXiv:1703.09035}, 2017.

\bibitem[Chen and Yu(2019)]{DAO}
Xiong-Hui Chen and Yang Yu.
\newblock Reinforcement learning with derivative-free exploration.
\newblock In \emph{Proceedings of the 18th International Conference on
  Autonomous Agents and MultiAgent Systems}, pages 1880--1882, 2019.

\bibitem[Clerc(2010)]{PSO_Book}
Maurice Clerc.
\newblock \emph{Particle swarm optimization}, volume~93.
\newblock John Wiley \& Sons, 2010.

\bibitem[Colas et~al.(2018)Colas, Sigaud, and Oudeyer]{GEP-PG}
C{\'e}dric Colas, Olivier Sigaud, and Pierre-Yves Oudeyer.
\newblock {GEP-PG}: Decoupling exploration and exploitation in deep
  reinforcement learning algorithms.
\newblock In \emph{International conference on machine learning}, pages
  1039--1048. PMLR, 2018.

\bibitem[Cully et~al.(2015)Cully, Clune, Tarapore, and Mouret]{QDNature}
Antoine Cully, Jeff Clune, Danesh Tarapore, and Jean-Baptiste Mouret.
\newblock Robots that can adapt like animals.
\newblock \emph{Nature}, 521\penalty0 (7553):\penalty0 503--507, 2015.

\bibitem[De~Boer et~al.(2005)De~Boer, Kroese, Mannor, and Rubinstein]{CEM}
Pieter-Tjerk De~Boer, Dirk~P Kroese, Shie Mannor, and Reuven~Y Rubinstein.
\newblock A tutorial on the cross-entropy method.
\newblock \emph{Annals of operations research}, 134\penalty0 (1):\penalty0
  19--67, 2005.

\bibitem[Fujimoto et~al.(2018)Fujimoto, Hoof, and Meger]{TD3}
Scott Fujimoto, Herke Hoof, and David Meger.
\newblock Addressing function approximation error in actor-critic methods.
\newblock In \emph{International Conference on Machine Learning}, pages
  1587--1596. PMLR, 2018.

\bibitem[Gangwani and Peng(2017)]{GPO}
Tanmay Gangwani and Jian Peng.
\newblock Policy optimization by genetic distillation.
\newblock \emph{arXiv preprint arXiv:1711.01012}, 2017.

\bibitem[Grefenstette et~al.(2011)Grefenstette, Moriarty, and
  Schultz]{EARL2011}
John~J Grefenstette, David~E Moriarty, and Alan~C Schultz.
\newblock Evolutionary algorithms for reinforcement learning.
\newblock \emph{arXiv e-prints}, pages arXiv--1106, 2011.

\bibitem[Haarnoja et~al.(2018)Haarnoja, Zhou, Hartikainen, Tucker, Ha, Tan,
  Kumar, Zhu, Gupta, Abbeel, et~al.]{SAC_application}
Tuomas Haarnoja, Aurick Zhou, Kristian Hartikainen, George Tucker, Sehoon Ha,
  Jie Tan, Vikash Kumar, Henry Zhu, Abhishek Gupta, Pieter Abbeel, et~al.
\newblock Soft actor-critic algorithms and applications.
\newblock \emph{arXiv preprint arXiv:1812.05905}, 2018.

\bibitem[Hassan et~al.(2005)Hassan, Cohanim, De~Weck, and Venter]{PSO_GA}
Rania Hassan, Babak Cohanim, Olivier De~Weck, and Gerhard Venter.
\newblock A comparison of particle swarm optimization and the genetic
  algorithm.
\newblock In \emph{46th AIAA/ASME/ASCE/AHS/ASC structures, structural dynamics
  and materials conference}, page 1897, 2005.

\bibitem[Kalashnikov et~al.(2018)Kalashnikov, Irpan, Pastor, Ibarz, Herzog,
  Jang, Quillen, Holly, Kalakrishnan, Vanhoucke, et~al.]{Qt-Opt}
Dmitry Kalashnikov, Alex Irpan, Peter Pastor, Julian Ibarz, Alexander Herzog,
  Eric Jang, Deirdre Quillen, Ethan Holly, Mrinal Kalakrishnan, Vincent
  Vanhoucke, et~al.
\newblock Scalable deep reinforcement learning for vision-based robotic
  manipulation.
\newblock In \emph{Conference on Robot Learning}, pages 651--673. PMLR, 2018.

\bibitem[Kennedy and Eberhart(1995)]{PSO}
James Kennedy and Russell Eberhart.
\newblock Particle swarm optimization.
\newblock In \emph{Proceedings of ICNN'95-international conference on neural
  networks}, volume~4, pages 1942--1948. IEEE, 1995.

\bibitem[Khadka and Tumer(2018)]{ERL}
Shauharda Khadka and Kagan Tumer.
\newblock Evolution-guided policy gradient in reinforcement learning.
\newblock In \emph{Proceedings of the 32nd International Conference on Neural
  Information Processing Systems}, pages 1196--1208, 2018.

\bibitem[Khadka et~al.(2019)Khadka, Majumdar, Nassar, Dwiel, Tumer, Miret, Liu,
  and Tumer]{CERL}
Shauharda Khadka, Somdeb Majumdar, Tarek Nassar, Zach Dwiel, Evren Tumer,
  Santiago Miret, Yinyin Liu, and Kagan Tumer.
\newblock Collaborative evolutionary reinforcement learning.
\newblock In \emph{International Conference on Machine Learning}, pages
  3341--3350. PMLR, 2019.

\bibitem[Lee et~al.(2020)Lee, Lee, Shin, and Kweon]{AES-RL}
Kyunghyun Lee, Byeong-Uk Lee, Ukcheol Shin, and In~So Kweon.
\newblock An efficient asynchronous method for integrating evolutionary and
  gradient-based policy search.
\newblock \emph{arXiv preprint arXiv:2012.05417}, 2020.

\bibitem[Lillicrap et~al.(2015)Lillicrap, Hunt, Pritzel, Heess, Erez, Tassa,
  Silver, and Wierstra]{rllab}
Timothy~P Lillicrap, Jonathan~J Hunt, Alexander Pritzel, Nicolas Heess, Tom
  Erez, Yuval Tassa, David Silver, and Daan Wierstra.
\newblock Continuous control with deep reinforcement learning.
\newblock \emph{arXiv preprint arXiv:1509.02971}, 2015.

\bibitem[Lorenzo et~al.(2017)Lorenzo, Nalepa, Ramos, and Pastor]{PSO_Hyper2}
Pablo~Ribalta Lorenzo, Jakub Nalepa, Luciano~Sanchez Ramos, and
  Jos{\'e}~Ranilla Pastor.
\newblock Hyper-parameter selection in deep neural networks using parallel
  particle swarm optimization.
\newblock In \emph{Proceedings of the Genetic and Evolutionary Computation
  Conference Companion}, pages 1864--1871, 2017.

\bibitem[Majid et~al.(2021)Majid, Saaybi, van Rietbergen, Francois-Lavet,
  Prasad, and Verhoeven]{DRLvsES}
Amjad~Yousef Majid, Serge Saaybi, Tomas van Rietbergen, Vincent Francois-Lavet,
  R~Venkatesha Prasad, and Chris Verhoeven.
\newblock Deep reinforcement learning versus evolution strategies: A
  comparative survey.
\newblock \emph{arXiv preprint arXiv:2110.01411}, 2021.

\bibitem[Marchesini et~al.(2021)Marchesini, Corsi, and Farinelli]{Super-RL}
Enrico Marchesini, Davide Corsi, and Alessandro Farinelli.
\newblock Genetic soft updates for policy evolution in deep reinforcement
  learning.
\newblock In \emph{International Conference on Learning Representations}, 2021.

\bibitem[Mitchell(1998)]{GA}
Melanie Mitchell.
\newblock \emph{An introduction to genetic algorithms}.
\newblock MIT press, 1998.

\bibitem[Mnih et~al.(2013)Mnih, Kavukcuoglu, Silver, Graves, Antonoglou,
  Wierstra, and Riedmiller]{c:Atari}
Volodymyr Mnih, Koray Kavukcuoglu, David Silver, Alex Graves, Ioannis
  Antonoglou, Daan Wierstra, and Martin Riedmiller.
\newblock Playing atari with deep reinforcement learning.
\newblock \emph{arXiv preprint arXiv:1312.5602}, 2013.

\bibitem[Mohsen et~al.(2012)Mohsen, Hadhoud, Moustafa, and Ameen]{PSO_Image2}
Fahd Mohsen, Mohiy~M Hadhoud, Kamel Moustafa, and Khalid Ameen.
\newblock A new image segmentation method based on particle swarm optimization.
\newblock \emph{Int. Arab J. Inf. Technol.}, 9\penalty0 (5):\penalty0 487--493,
  2012.

\bibitem[Moriarty et~al.(1999)Moriarty, Schultz, and Grefenstette]{EARL1999}
David~E Moriarty, Alan~C Schultz, and John~J Grefenstette.
\newblock Evolutionary algorithms for reinforcement learning.
\newblock \emph{Journal of Artificial Intelligence Research}, 11:\penalty0
  241--276, 1999.

\bibitem[Nair et~al.(2018)Nair, McGrew, Andrychowicz, Zaremba, and
  Abbeel]{Q-filter}
Ashvin Nair, Bob McGrew, Marcin Andrychowicz, Wojciech Zaremba, and Pieter
  Abbeel.
\newblock Overcoming exploration in reinforcement learning with demonstrations.
\newblock In \emph{2018 IEEE International Conference on Robotics and
  Automation (ICRA)}, pages 6292--6299. IEEE, 2018.

\bibitem[Omran et~al.(2006)Omran, Salman, and Engelbrecht]{PSO_Image}
Mahamed~GH Omran, Ayed Salman, and Andries~P Engelbrecht.
\newblock Dynamic clustering using particle swarm optimization with application
  in image segmentation.
\newblock \emph{Pattern Analysis and Applications}, 8\penalty0 (4):\penalty0
  332--344, 2006.

\bibitem[Pierrot et~al.(2022)Pierrot, Mac{\'e}, Chalumeau, Flajolet, Cideron,
  Beguir, Cully, Sigaud, and Perrin-Gilbert]{QD-RL}
Thomas Pierrot, Valentin Mac{\'e}, Felix Chalumeau, Arthur Flajolet, Geoffrey
  Cideron, Karim Beguir, Antoine Cully, Olivier Sigaud, and Nicolas
  Perrin-Gilbert.
\newblock Diversity policy gradient for sample efficient quality-diversity
  optimization.
\newblock In \emph{ICLR Workshop on Agent Learning in Open-Endedness}, 2022.

\bibitem[Pourchot and Sigaud(2019)]{CEM-RL}
Alo{\"\i}s Pourchot and Olivier Sigaud.
\newblock {CEM-RL}: Combining evolutionary and gradient-based methods for
  policy search.
\newblock In \emph{International Conference on Learning Representations}, 2019.

\bibitem[Pugh et~al.(2016)Pugh, Soros, and Stanley]{QD}
Justin~K Pugh, Lisa~B Soros, and Kenneth~O Stanley.
\newblock Quality diversity: A new frontier for evolutionary computation.
\newblock \emph{Frontiers in Robotics and AI}, 3:\penalty0 40, 2016.

\bibitem[Qian and Yu(2021)]{erl_review}
Hong Qian and Yang Yu.
\newblock Derivative-free reinforcement learning: A review.
\newblock \emph{arXiv preprint arXiv:2102.05710}, 2021.

\bibitem[Salimans et~al.(2017)Salimans, Ho, Chen, Sidor, and Sutskever]{ES}
Tim Salimans, Jonathan Ho, Xi~Chen, Szymon Sidor, and Ilya Sutskever.
\newblock Evolution strategies as a scalable alternative to reinforcement
  learning.
\newblock \emph{arXiv preprint arXiv:1703.03864}, 2017.

\bibitem[Schulman et~al.(2017)Schulman, Wolski, Dhariwal, Radford, and
  Klimov]{PPO}
John Schulman, Filip Wolski, Prafulla Dhariwal, Alec Radford, and Oleg Klimov.
\newblock Proximal policy optimization algorithms.
\newblock \emph{arXiv preprint arXiv:1707.06347}, 2017.

\bibitem[S.Fujimoto(2018)]{TD3_Code}
S.Fujimoto.
\newblock Open-source implementation for {TD3}.
\newblock \url{https://github.com/sfujim/TD3}, 2018.

\bibitem[Shao et~al.(2022)Shao, You, Yan, Yuan, Sun, and Bohg]{GRAC}
Lin Shao, Yifan You, Mengyuan Yan, Shenli Yuan, Qingyun Sun, and Jeannette
  Bohg.
\newblock Grac: Self-guided and self-regularized actor-critic.
\newblock In \emph{Conference on Robot Learning}, pages 267--276. PMLR, 2022.

\bibitem[Shi and Singh(2021)]{SAC-CEPO}
Zhenyang Shi and Surya~PN Singh.
\newblock Soft actor-critic with cross-entropy policy optimization.
\newblock \emph{arXiv preprint arXiv:2112.11115}, 2021.

\bibitem[Sigaud(2022)]{erl_survey}
Olivier Sigaud.
\newblock Combining evolution and deep reinforcement learning for policy
  search: a survey.
\newblock \emph{arXiv preprint arXiv:2203.14009}, 2022.

\bibitem[Silver et~al.(2014)Silver, Lever, Heess, Degris, Wierstra, and
  Riedmiller]{DPG}
David Silver, Guy Lever, Nicolas Heess, Thomas Degris, Daan Wierstra, and
  Martin Riedmiller.
\newblock Deterministic policy gradient algorithms.
\newblock In \emph{International conference on machine learning}, pages
  387--395. PMLR, 2014.

\bibitem[Silver et~al.(2016)Silver, Huang, Maddison, Guez, Sifre, Van
  Den~Driessche, Schrittwieser, Antonoglou, Panneershelvam, Lanctot,
  et~al.]{c:Go}
David Silver, Aja Huang, Chris~J Maddison, Arthur Guez, Laurent Sifre, George
  Van Den~Driessche, Julian Schrittwieser, Ioannis Antonoglou, Veda
  Panneershelvam, Marc Lanctot, et~al.
\newblock Mastering the game of go with deep neural networks and tree search.
\newblock \emph{nature}, 529\penalty0 (7587):\penalty0 484--489, 2016.

\bibitem[Simmons-Edler et~al.(2019)Simmons-Edler, Eisner, Mitchell, Seung, and
  Lee]{CGP}
Riley Simmons-Edler, Ben Eisner, Eric Mitchell, Sebastian Seung, and Daniel
  Lee.
\newblock Q-learning for continuous actions with cross-entropy guided policies.
\newblock \emph{arXiv preprint arXiv:1903.10605}, 2019.

\bibitem[Such et~al.(2017)Such, Madhavan, Conti, Lehman, Stanley, and
  Clune]{neurevolution}
Felipe~Petroski Such, Vashisht Madhavan, Edoardo Conti, Joel Lehman, Kenneth~O
  Stanley, and Jeff Clune.
\newblock Deep neuroevolution: Genetic algorithms are a competitive alternative
  for training deep neural networks for reinforcement learning.
\newblock \emph{arXiv preprint arXiv:1712.06567}, 2017.

\bibitem[Suri(2021)]{ESAC}
Karush Suri.
\newblock Off-policy evolutionary reinforcement learning with maximum
  mutations, 2021.
\newblock URL
  \url{https://nbviewer.jupyter.org/github/karush17/karush17.github.io/blob/master/_pages/temp4.pdf}.

\bibitem[Tambouratzis(2016)]{PSO_NLP}
George Tambouratzis.
\newblock Applying {PSO} to natural language processing tasks: Optimizing the
  identification of syntactic phrases.
\newblock In \emph{2016 IEEE Congress on Evolutionary Computation (CEC)}, pages
  1831--1838. IEEE, 2016.

\bibitem[Todorov et~al.(2012)Todorov, Erez, and Tassa]{Mujoco}
Emanuel Todorov, Tom Erez, and Yuval Tassa.
\newblock Mujoco: A physics engine for model-based control.
\newblock In \emph{2012 IEEE/RSJ International Conference on Intelligent Robots
  and Systems}, pages 5026--5033. IEEE, 2012.

\bibitem[Wang et~al.(2018)Wang, Tan, and Liu]{PSOapplication}
Dongshu Wang, Dapei Tan, and Lei Liu.
\newblock Particle swarm optimization algorithm: an overview.
\newblock \emph{Soft Computing}, 22\penalty0 (2):\penalty0 387--408, 2018.

\bibitem[Whiteson(2006)]{EARL2006}
Shimon Whiteson.
\newblock Evolutionary function approximation for reinforcement learning.
\newblock \emph{Journal of Machine Learning Research}, 7, 2006.

\bibitem[Ye(2017)]{PSO_Hyper}
Fei Ye.
\newblock Particle swarm optimization-based automatic parameter selection for
  deep neural networks and its applications in large-scale and high-dimensional
  data.
\newblock \emph{PloS one}, 12\penalty0 (12):\penalty0 e0188746, 2017.

\bibitem[Zang et~al.(2019)Zang, Qi, Yang, Liu, Zhang, Liu, and Sun]{PSO_NLP2}
Yuan Zang, Fanchao Qi, Chenghao Yang, Zhiyuan Liu, Meng Zhang, Qun Liu, and
  Maosong Sun.
\newblock Word-level textual adversarial attacking as combinatorial
  optimization.
\newblock \emph{arXiv preprint arXiv:1910.12196}, 2019.

\end{thebibliography}
\bibliographystyle{plainnat}

\newpage
\appendix
\section{Background}\label{Appendix:Background}
\subsection{Twin Delayed DDPG (TD3)}\label{sec:2.1}
The general off-policy RL method is Deep Deterministic Policy Gradient (DDPG)~\cite{DDPG}, which is developed for handling the case of high-dimensional continuous action space. ~\citet{TD3} extended DDPG to Twin Delayed DDPG (TD3). With a significant improvement upon DDPG, TD3 is a state-of-the-art off-policy algorithm for RL in continuous action spaces. TD3 uses an actor-critic architecture and maintains a deterministic policy (actor) $\pi:S\rightarrow A$, and two independent action-value function approximations (critics) $Q:S\times A\rightarrow Q(s,a)$. Generated by the actor performing the action $a$, the experience $(s_t,a_t,r_t,s_{t+1})$ is saved into a replay buffer $R$ thereafter. Actor and critic networks are updated by randomly sampling mini-batch from $R$. Critic is trained by minimizing the MSE loss function:
\begin{equation}
\begin{split}
L=&N^{-1}\sum(y-Q_{\theta_i}(s,a))^2 \\
where~~~&y=r+\gamma min_{i=1,2}Q_{\theta_i}(s',\widetilde{a})
\end{split}
\label{eq:td3_q}
\end{equation}
$\tilde{a}$ is the noisy action computed by adding Gaussian noise. TD3 adds noise to the target action. Actor is trained by deterministic policy gradient~\cite{DPG}:
\begin{equation}
\nabla_\theta J(\theta)=N^{-1}\sum\nabla_a Q_{\theta_1}(s,a)|_{a=\pi_\theta(s)}\nabla_\theta\pi_\theta(s)
\label{eq:td3_pi}
\end{equation}
\subsection{Particle Swarm Optimization}
PSO~\cite{PSO} is a meta-heuristic global optimization paradigm whose basic concept originated from the study of bird flocks foraging behavior.
It has turned out to be successful in dealing with diverse problems such as image segmentation~\cite{PSO_Image, PSO_Image2}, hyperparameter selection~\cite{PSO_Hyper, PSO_Hyper2} and natural language processing~\cite{PSO_NLP, PSO_NLP2}.
\citet{PSO_GA} has claimed the proof of equal effectiveness but superior efficiency for PSO over the Genetic Algorithm (GA)~\cite{GA}.
PSO uses a population to explore the optimal solution of the $D$-dimensional search space.
This population is called a swarm, which contains multiple individuals called particles.
Each particle has only two specified attributes: position and velocity.
One represents the current position in the search space, while the other represents the direction of movement.
\section{Proof of Proposition 1}
\setcounter{theorem}{0}
\label{Appendix:Proof}
\begin{proposition}
\label{proposition_proof:1}

EAS policy $\mu_e$ optimizes the action $a$ (generated by $\mu_\theta$) in the direction of increasing $Q_{\mu_\theta}$ (estimated by the critic) to obtain $a^e$.
Hence, for arbitrary state $s$, there exists inequality $\xi$: $\mathbb{E}_{a^e \sim\mu_e(s)}[Q_{\mu_\theta}  (s,a^e)] \ge \mathbb{E}_{a\sim\mu_\theta(s)}[Q_{\mu_\theta} (s,a)]$. Then, we hold that:

\begin{equation*}
\mathbb{E}_{a^e\sim\mu_e(s)}[Q_{\mu_e}(s,a^e)]\ge\mathbb{E}_{a\sim\mu_\theta(s)}[Q_{\mu_\theta}(s,a)]
\label{eq:15}
\end{equation*}
\end{proposition}

\begin{proof}
For arbitrary $s$, 
\vskip -0.2in
\begin{equation*}
\begin{aligned}
   ~V_{\mu_\theta}(s)
   &=\mathbb{E}_{a \sim \mu_\theta(s)}\left[Q_{\mu_\theta}(s,a)\right]
   \le\mathbb{E}_{a^e \sim \mu_e(s)}\left[Q_{\mu_\theta}(s,a^e)\right] \quad \\
   &=\mathbb{E}_{a^e \sim \mu_e(s)}\left[r_{\mu_e}+\gamma \mathbb{E}_{s' \sim p, a'\sim \mu_\theta(s')}\left[Q_{\mu_\theta}(s',a')\right]\right] \\
   &\le\mathbb{E}_{a^e \sim \mu_e(s)}\left[r_{\mu_e}+\gamma \mathbb{E}_{s' \sim p, {a^{e}}'\sim \mu_e(s')}\left[Q_{\mu_\theta}(s',{a^{e}}')\right]\right] \\
   &=\mathbb{E}_{a^e\sim \mu_e(s)}\left[r_{\mu_e}+\gamma r'_{\mu_e}+\gamma^2 \mathbb{E}_{s'' \sim p, a'' \sim \mu_\theta(s'')}\left[Q_{\mu_\theta}(s'',a'')\right]\right] \\
   &\le \mathbb{E}_{a^e\sim \mu_e(s)}\left[r_{\mu_e}+\gamma r'_{\mu_e}+\gamma^2                    \mathbb{E}_{s'' \sim p, {a^e}''\sim \mu_e(s'')}\left[Q_{\mu_\theta}(s'',{a^e}'')\right]\right] \\
   &\cdots\\
   &\le \mathbb{E}_{a^e\sim \mu_e(s)}\left[r_{\mu_e}+\gamma r'_{\mu_e}+\gamma^2 r^{''}_{\mu_e}+...+\gamma^\infty  r_{\mu_e}^\infty\right] \\
   &=\mathbb{E}_{a^e\sim \mu_e(s)}\left[\sum\nolimits_{i=0}^\infty \gamma^{i}r^i_{\mu_e}\right]~ (r^i_{\mu_e}=\mathbb{E}_{s^i \sim p, a^{e i} \sim \mu_e} [r(s^i,a^{e i})])\\
   &\cong \mathbb{E}_{a^e \sim \mu_e(s)}\left[Q_{\mu_e}(s,a^e)\right]  
   =V_{\mu_e}(s)
\end{aligned}    
\label{eq:16}
\end{equation*}
\vskip -0.25in
\end{proof}
\clearpage
\section{Implementation Details}\label{Appendix:Implementation details}

\subsection{Fitness Evaluator}\label{sec:3.1}
PSO requires a fitness evaluator to evaluate the performance of actions in each generation.
We hope the action after evolution can obtain a greater expected reward.
So the average long-term reward of performing evolutionary actions in the environment appears to be an insightful fitness evaluator.
However, this will significantly slow down the evaluation efficiency, which is not what we expect.
Therefore, it is decisive to employ a fitness evaluator that can efficiently evaluate the value of each action.

In RL, we use the state-action value function $Q(s, a)$ to represent the long-term reward.
$Q(s,a)$ is denoted by Eq.~\ref{eq:5}, meaning the expected reward the policy $\mu_\theta$ can obtain by performing action $a_t$ in state $s_t$.
Once the $Q$ value corresponding to the action is known, we can directly evaluate the expected reward of the action without actually executing it in the environment.
Thus, the state-action value function seems to be a more sensible fitness evaluator.

Eq.~\ref{eq:6} is a recursive version of Eq.~\ref{eq:5}, also known as the Bellman expectation equation, which resembles dynamic programming and makes it feasible to estimate $Q$.
For specific RL algorithms, a value network is used to estimate the $Q$ function, which is called critic in TD3. Ultimately, we decided to choose the critic network as the fitness evaluator. The action in population will be input into the critic network to obtain the $Q$ value, whose magnitude can represent the performance of action.
\begin{equation}\label{eq:5}
Q_{\mu_\theta}(s_t,a_t)=\mathbb{E}_{\mu_\theta}\left[\left.\sum_{t'=t}^T r(s_{t'}, a_{t'})\right\vert s_t,a_t\right]
\end{equation}
\begin{equation}\label{eq:6}
\begin{split}
=\mathbb{E}_{s_{t+1}\sim p}\left[r(s_t,a_t)
+\mathbb{E}_{a_{t+1}\sim\mu_\theta}[Q_{\mu_\theta}(s_{t+1},a_{t+1})]\right] 
\end{split}
\end{equation}
\subsection{EAS-TD3 Pseudocode}\label{Appendix:EAS-TD3 Pseudocode}
\begin{algorithm}[htb]
\caption{EAS-TD3}
\label{alg:EAS-TD3}
\begin{algorithmic}[1] 
\State Initialize policy networks $\mu_\theta$, critic networks, $Q_{\mu_\theta1}$,$Q_{\mu_\theta2}$ with random parameters.
\State Initialize Archive $\mathcal{A}$, Replay Buffer $\mathcal{R}$
\For{$t=1~\textbf{to}~total\_steps$}
\State Policy $\mu_\theta$ chooses an action $a_t$ based on the current state $s_t$
\State Execute action $a_t$, obtain reward $r_t$, next state $s_{t+1}$, terminal flag $d$
\State Store sample tuple $(s_t, a_t, r_t, s_{t+1}, d)$ in Buffer $\mathcal{R}$.
\State \textcolor{red}{\emph{Utilize EAS to evolve action $a_t$, and obtain evolutionary action $a^e$}}
\State \textcolor{red}{\emph{Store state-action pair $(s_t, a^e)$ in Archive $\mathcal{A}$}}
\State Sample a batch of samples ${(s, a, r, s', d)}$ from Buffer $\mathcal{R}$
\State Policy $\mu_\theta$ chooses next action $a'$ based on next state $s'$
\State $Q(s',a')\gets min_{j=1,2}(Q_{\mu_\theta j}(s', a'))$
\State $y \gets r + \gamma(1-d)Q(s',a')$
\State Update the critic $Q_{\mu_\theta1}$, $Q_{\mu_\theta2}$ according to Bellman loss~Eq.\ref{eq:td3_q}
\If{$t~\%~2 == 0$}
\State Update the policy $\mu_\theta$ according to policy gradient ~Eq.\ref{eq:td3_pi}
\State \textcolor{red}{\emph{Sample a batch of state-action pairs ${(s, a^e)}$ from Archive $\mathcal{A}$}}
\State \textcolor{red}{\emph{Update the policy $\mu_\theta$ according to evolutionary action gradient ~Eq.\ref{eq:11}}}
\EndIf
\EndFor
\end{algorithmic}
\end{algorithm}
\clearpage

\section{Hyperparameters}
\label{Appendix:Hyperparameters}
%
Hyperparameters of TD3 and EAS are detailed in Table~\ref{tab:2}.
In addition, for Delayed Halfcheetah, the learning rate is still 1e-3.
For Delayed Walker2d, the learning rate is 1e-3, the archive size is 10,000.
For Delayed Humanoid, the learning rate is 1e-4.
For Delayed Ant, the learning rate is 1e-4 and the replay buffer size is 200,000.

\newcommand{\tabincell}[2]{\begin{tabular}{@{}#1@{}}#2\end{tabular}}
\begin{table}[htb]
\vskip -0.2in
\centering
\caption{Hyperparameters of TD3 and EAS in MuJoCo environments}
\vskip 0.15in
\resizebox{.78\columnwidth}{!}{
    \begin{tabular}{lcc}
        \toprule
            \textbf{Hyperparameter} & \tabincell{c}{\textbf{All environments except} \\ \textbf{for Humanoid-v3}} & \textbf{Humanoid-v3} \\
        \midrule
        Hyperparameters for TD3 \\
        ~~~~Batch size & 100 & 256 \\
        ~~~~Policy \& Critic network & (400,300) & (256,256) \\
        ~~~~Learning rate & 1e-3 & 3e-4 \\
        ~~~~Optimizer & \multicolumn{2}{c}{Adam} \\
        ~~~~Replay buffer size & \multicolumn{2}{c}{1e6} \\
        ~~~~Start timesteps &\multicolumn{2}{c}{2.5e4} \\
        ~~~~Exploration noise & \multicolumn{2}{c}{$\mathcal{N}(0,0.1)$} \\
        ~~~~Discount factor & \multicolumn{2}{c}{0.99}\\
        ~~~~Target update rate & \multicolumn{2}{c}{5e-3} \\
        ~~~~Policy update freq & \multicolumn{2}{c}{2} \\
        \midrule
        Hyperparameters for EAS \\
        ~~~~Archive size & 
        \multicolumn{2}{c}{1e5} \\
        ~~~~Population size $N$ & \multicolumn{2}{c}{10} \\
        ~~~~Inertia weight $\omega$ & \multicolumn{2}{c}{1.2} \\
        ~~~~Acceleration coefficients $c_1$ \& $c_2$ & \multicolumn{2}{c}{1.5} \\
        ~~~~Random coefficients $r_1$ \& $r_2$ & \multicolumn{2}{c}{$rand (0,1)$} \\
        ~~~~Iteration number & \multicolumn{2}{c}{10} \\
        ~~~~Maximum velocity $v_{max}$ & \multicolumn{2}{c}{0.1} \\
        \bottomrule
    \end{tabular}
}
\vskip -0.10in
\label{tab:2}
\end{table}
\begin{itemize}
    \item \textbf{Archive size = 100,000} \\
    Archive stores state-action pairs $(s,a^e)$ corresponding to evolutionary actions.
    As shown in the ablation study, this parameters should not be too large or too small, which is usually set to one-tenth of the replay buffer size.
    \item \textbf{Population size $N$ = 10} \\
    Population size represents the number of actions in the population.
    A large number allows each iteration to cover a larger search space.
    However, more particles will increase the computational complexity of each iteration.
    10 is a moderate value.
    \item \textbf{Inertia weight $\omega$ = 1.2} \\
    Inertia weight is used in the Eq.~\ref{eq:velocity} for velocity update.
    Velocity determines the direction and magnitude of the action update.
    Inertia weight describes the influence of the previous generation's velocity on the current generation's velocity, which is usually set to between 0.8 and 1.2.
    If $\omega$ is large, the global optimization ability is strong, and the local optimization ability is weak and vice versa.
    \item \textbf{
    Acceleration coefficients $c_1$ = 1.5, $c_2$ = 1.5
    and random coefficients $r_1$ = \bm{$rand(0,1)$}, $r_2$ = \bm{$rand(0,1)$}
    }\\
    According to Eq.~\ref{eq:velocity}, acceleration coefficients $c_1$ and $c_2$, together with random coefficients $r_1$ and $r2$, control the stochastic influence of individual optimal action $p^b$ and global optimal action $g^b$ on the current generation's velocity.
    $c_1$ and $c_2$ are also referred to as trust parameters, where $c_1$ expresses how confident the current action trusts itself while $c_2$ expresses how confident the current action trusts the population~\cite{PSO_Book}.
    According to empirical studies, $c_1$, $c_2$ are generally set between 0.5 and 2.5.
    \item \textbf{Iteration number = 10} \\
    Iteration number represents the number of generations of the action population.
    \item \textbf{Maximum velocity $v_{max}$ = 0.1} \\
    $v_{max}$ represents the maximum velocity of each action in the population, which is generally set to one-tenth of the maximum value of the action.
\end{itemize}

\section{Additional Experiments on Fitness Evaluator}
During the training process, both the critic network and the policy network gradually get better.
Eventually, the critic network can accurately evaluate the $Q$ values of actions.
The policy network can output high-quality actions to obtain higher cumulative rewards.
It is a common bootstrap learning process in RL.

As mentioned in \cref{sec:3.1}, we regard the value network in RL (called critic network in TD3) as a fitness evaluator to evaluate the fitness of each action in the population.
In this section, we intend to investigate the influence of different fitness evaluators on our approach.
Specifically, we devised a new variant of EAS-TD3 called \textit{PretrainedCritic-EAS-TD3}.
We use TD3 to pretrain a critic network~(one million timesteps) as the fitness evaluator of EAS, whose parameters will be frozen during the training process.
Figure~\ref{fig:14} shows the learning curves, which indicates that each fitness evaluator has its own advantages and limitations.
In the early stage of training, the pretrained fitness evaluator provides a relatively accurate assessment of $Q$ values and EAS can generate high-quality evolutionary actions.
However, these high-quality evolutionary actions may lead to overexploitation and lack of exploration.
Although the fitness evaluator of EAS can not provide an accurate evaluation of $Q$ value and may weaken the quality of evolutionary actions in the early stage of training, the resulting evolutionary actions may contribute to the exploration of the policy.
In a word, the pretrained critic network might be an alternative for the fitness evaluator, and how to devise the fitness evaluator will be interesting follow-up work.
\begin{figure}[htbp]
\centering
    \begin{subfigure}[t]{0.24\textwidth}
    \centering
    \includegraphics[width=\columnwidth]{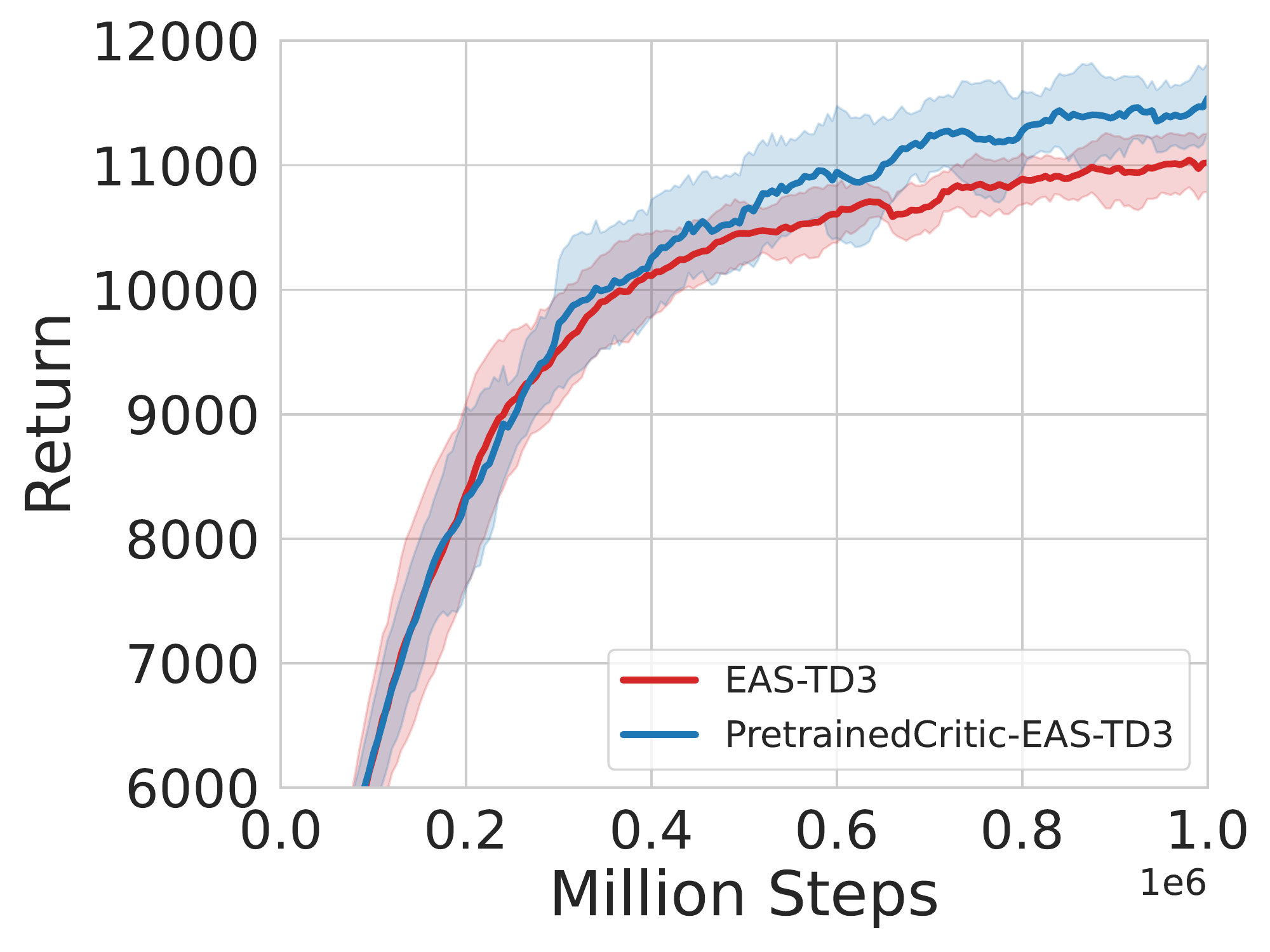}
    \caption{HalfCheetah-v3}
    \label{fig:14a}
    \end{subfigure}%
    \begin{subfigure}[t]{0.24\textwidth}
    \centering
    \includegraphics[width=\columnwidth]{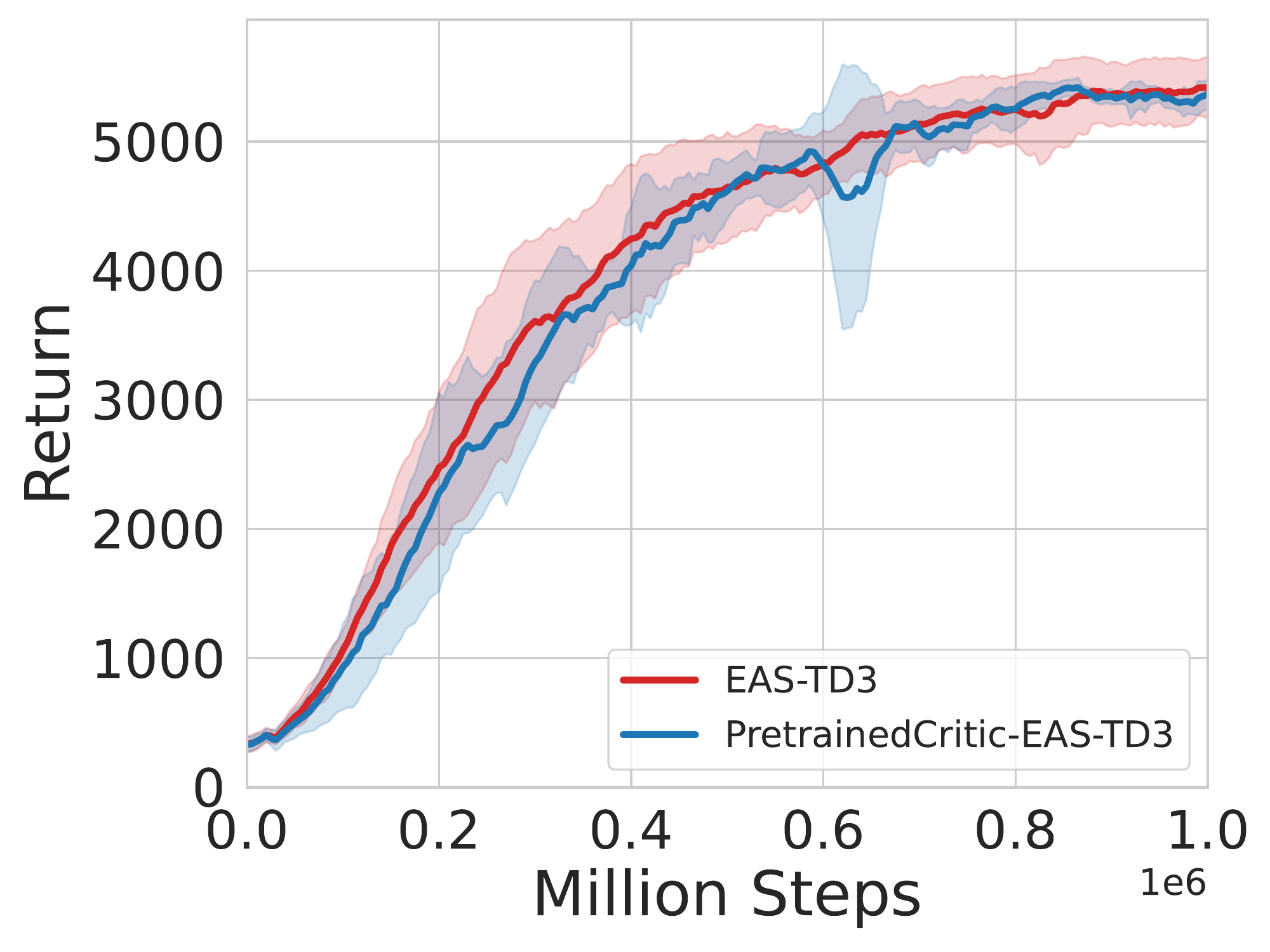}
    \caption{Ant-v3}
    \label{fig:14b}
    \end{subfigure}%
    \begin{subfigure}[t]{0.24\textwidth}
        \centering
        \includegraphics[width=\columnwidth]{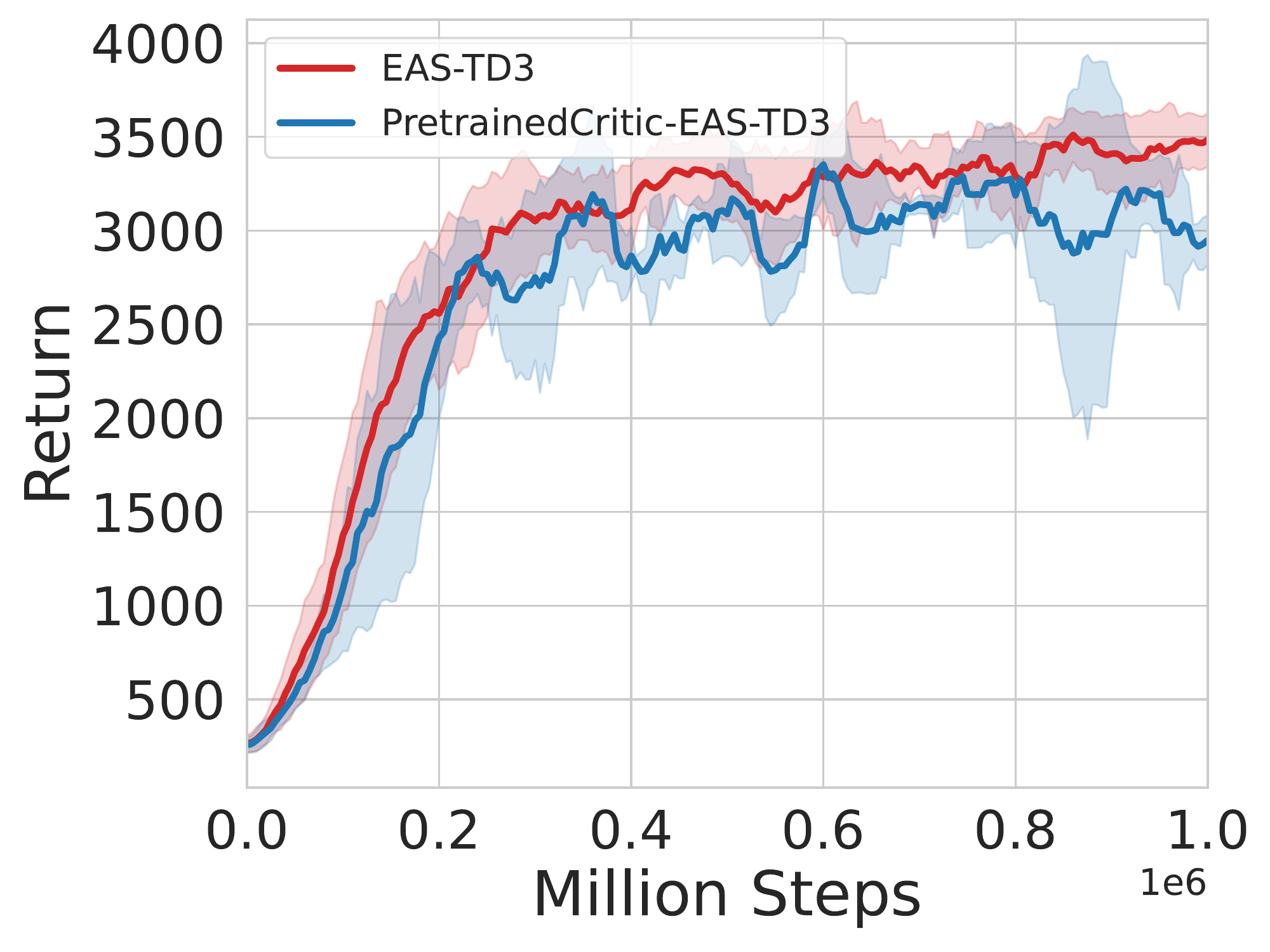}
    \caption{Hopper-v3}
    \label{fig:14c}
    \end{subfigure}
    \begin{subfigure}[t]{0.24\textwidth}
        \centering
        \includegraphics[width=\columnwidth]{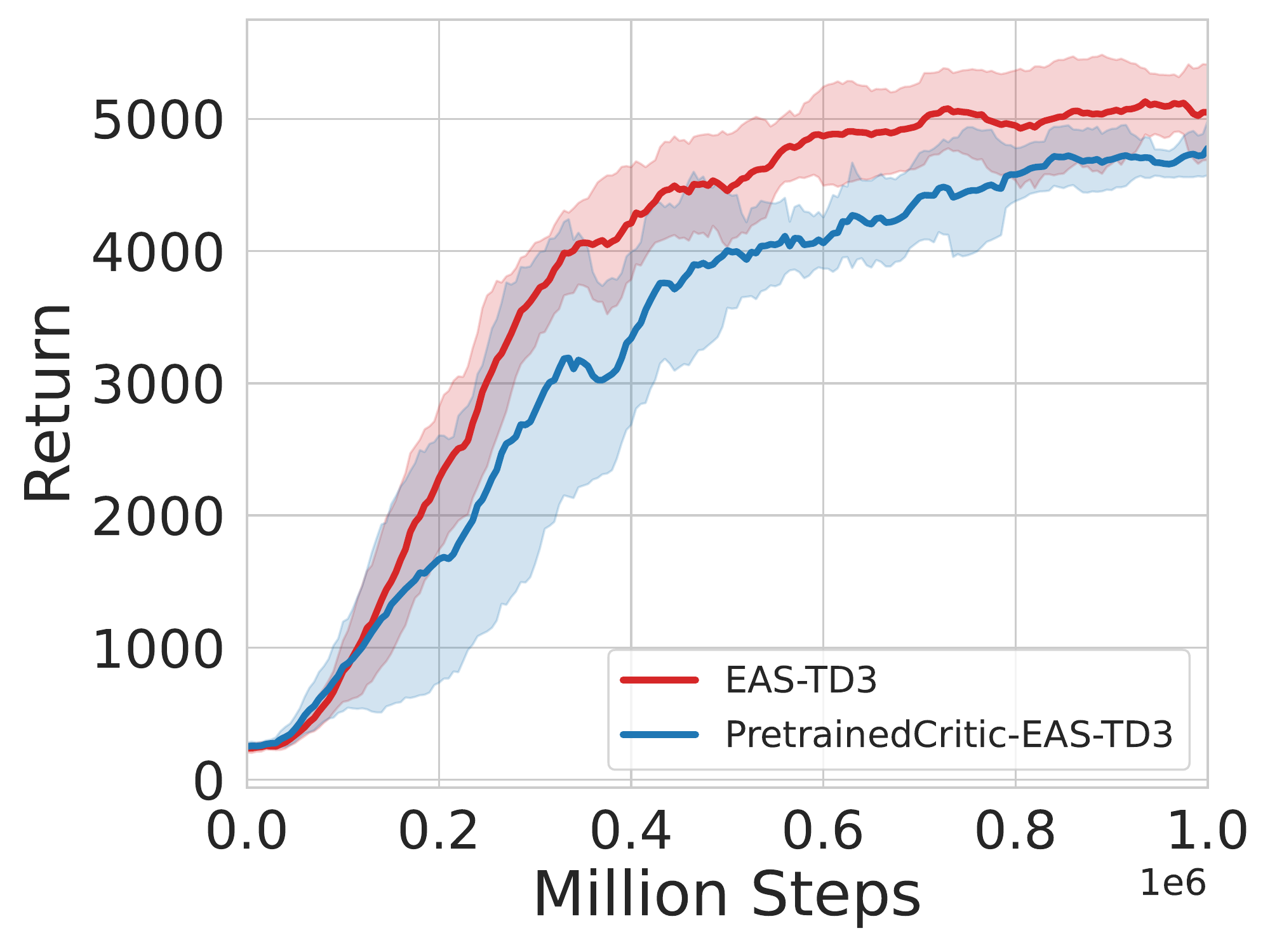}
    \caption{Walker2d-v3}
    \label{fig:14d}
    \end{subfigure}
\caption{Experiments~(mean $\pm$ standard deviation) on pretrained fitness evalautor.}
\label{fig:14}
    \vskip -0.2in
\end{figure}

\section{Additional Evaluation on Evolutionary Action}
\label{Appendix:More Evolutionary Action Evaluation}
Figures~\ref{fig:11},~\ref{fig:12},~\ref{fig:13} display the distribution of actions chosen by the policy and the corresponding evolutionary actions during the training process of Walker2d-v3 (from dimension 0 to dimension 5).
For presentation purposes, we rotate the figure to the vertical.
Each row represents the action space of every 100,000 timesteps.
In other words, we store all the actions chosen by the policy during training and the corresponding evolutionary actions and plot the distribution of the actions every 100,000 timesteps.
The red contour represents the current policy’s actions and the blue contour represents the corresponding evolutionary actions.
The performance of the policy will gradually improve over timesteps.
Therefore, the action distribution plot of each row in the figures indicates a higher expected reward than the previous row.
In Figure~\ref{fig:11}, we can see the evolutionary actions (the blue plot) show a trend of aggregation in the upper right corner at 500,000 to 600,000 timesteps.
Then the action distribution of the policy (the red plot)  gradually showed the same trend guided by evolutionary actions.
The same situation occurs in Figure~\ref{fig:13}, where evolutionary actions (the blue plot) show a trend of extending from the lower right to the left at 100,000 to 200,000 timesteps.
Then the action of the policy (the red plot)  gradually showed the same trend.
In short, evolutionary actions will predict the action space with a higher expected reward in advance and guide the current policy to move towards there.
Policy gradient and evolutionary gradient promote strategy learning in the same direction, ultimately improving the sample efficiency and the final performance.

\begin{figure*}[htb]
     \centering
     \begin{subfigure}[t]{0.50\textwidth}
         \centering
        \includegraphics[width=0.95\textwidth]{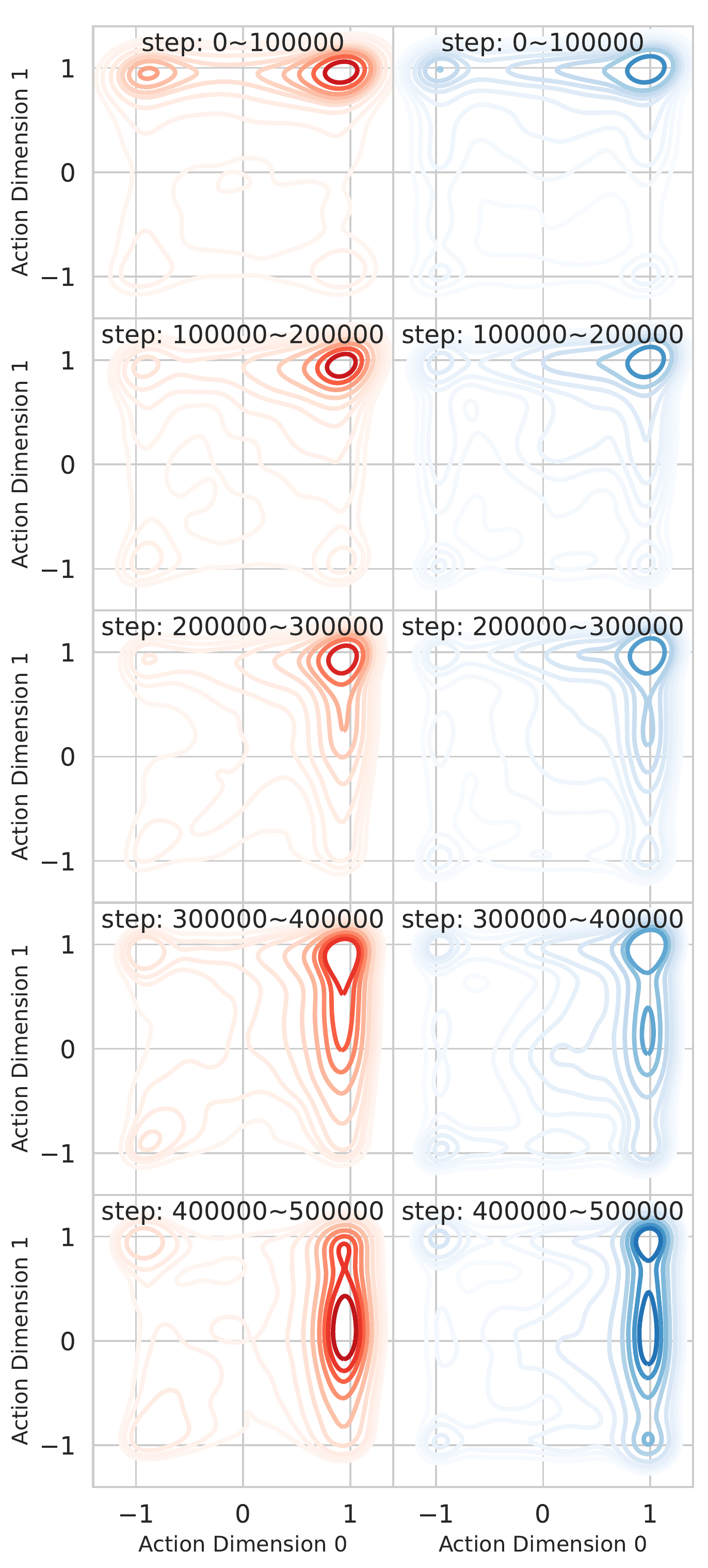}
         \label{fig:11a}
     \end{subfigure}%
     \hfill
     \begin{subfigure}[t]{0.50\textwidth}
         \centering
         \includegraphics[width=0.95\textwidth]{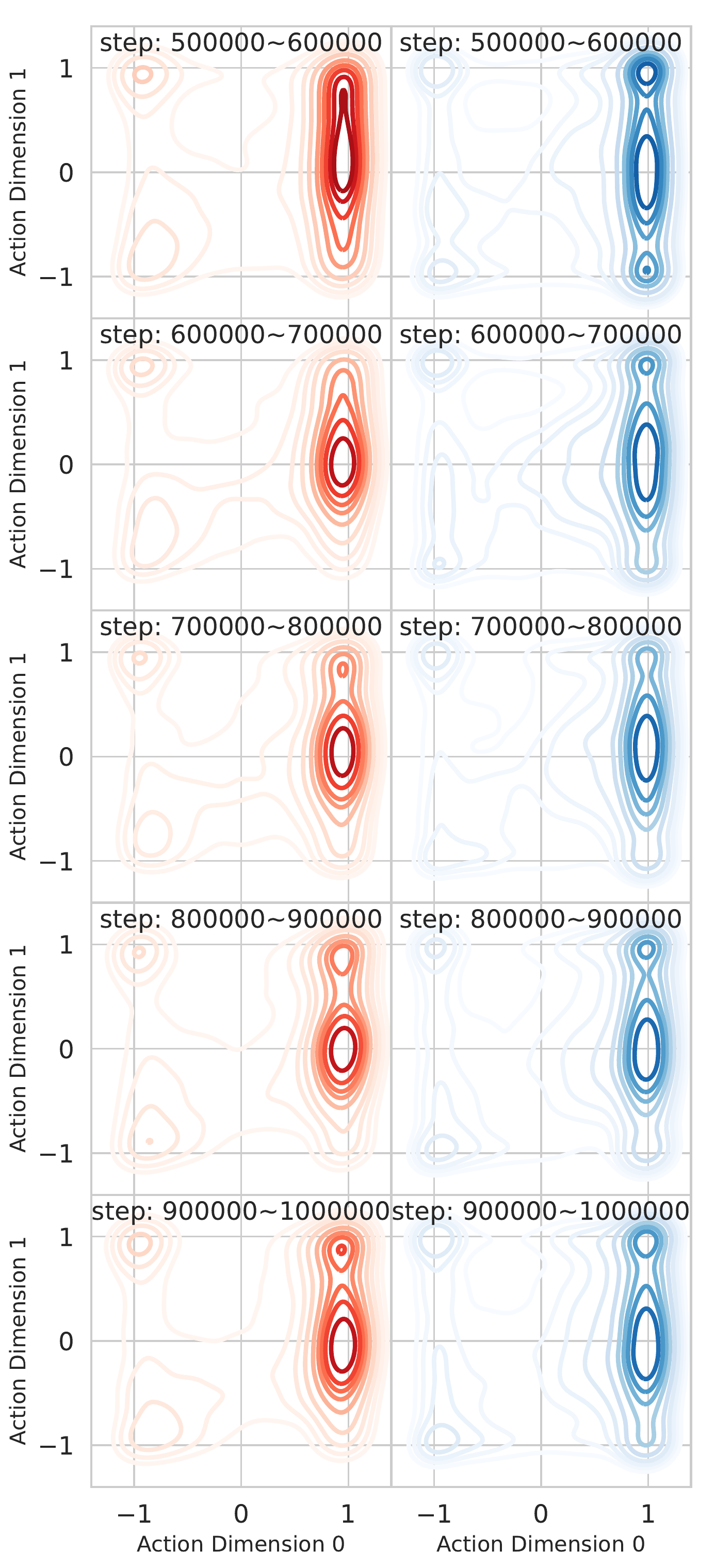}
         \label{fig:11b}
     \end{subfigure}%
     \hfill
    \caption{The zeroth and first dimensional distribution of the action in Walker2d-v3.
    The red contour represents the current policy’s actions and the blue contour represents the evolutionary actions.}
    \label{fig:11}
\end{figure*}

\begin{figure*}[htb]
     \centering
     \begin{subfigure}[t]{0.50\textwidth}
         \centering
        \includegraphics[width=0.95\textwidth]{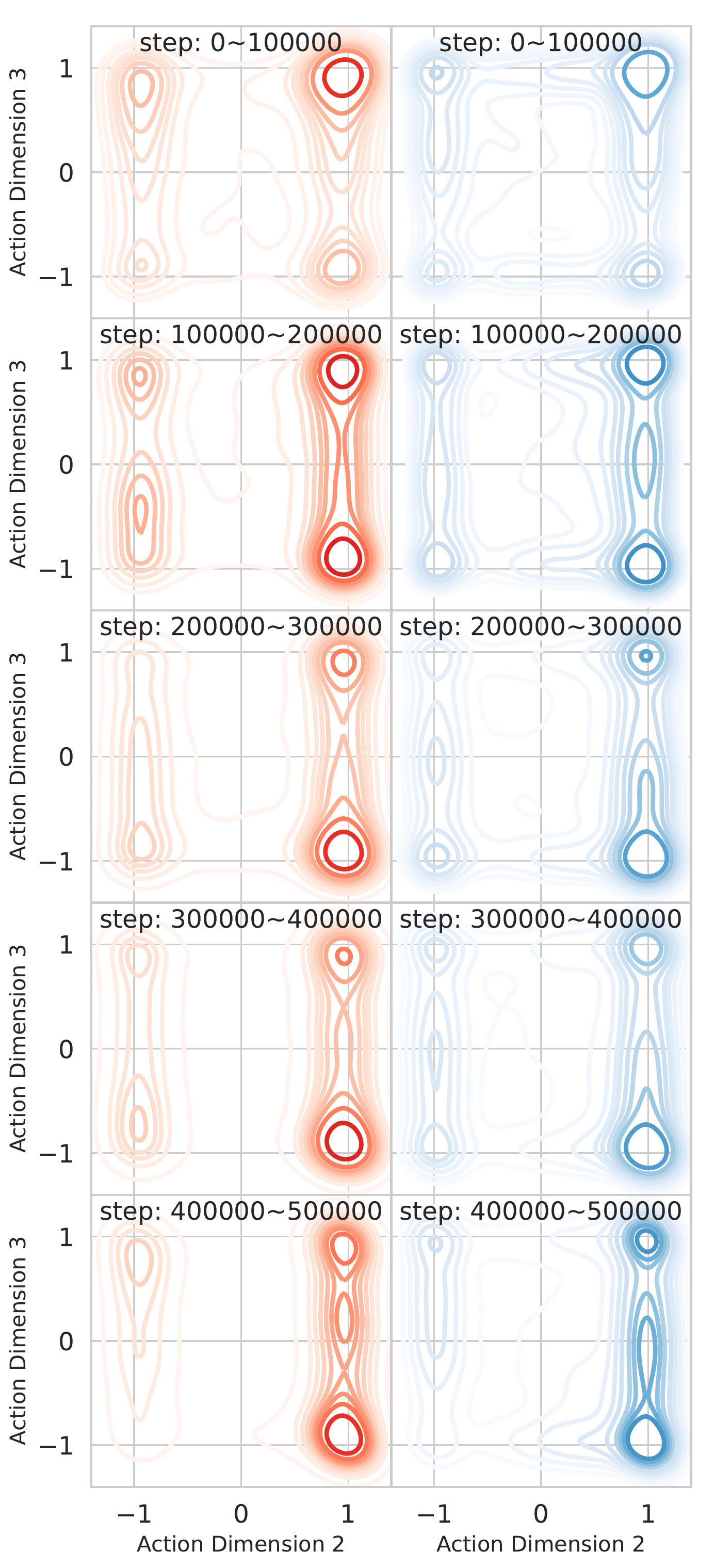}
         \label{fig:12a}
     \end{subfigure}%
     \hfill
     \begin{subfigure}[t]{0.50\textwidth}
         \centering
         \includegraphics[width=0.95\textwidth]{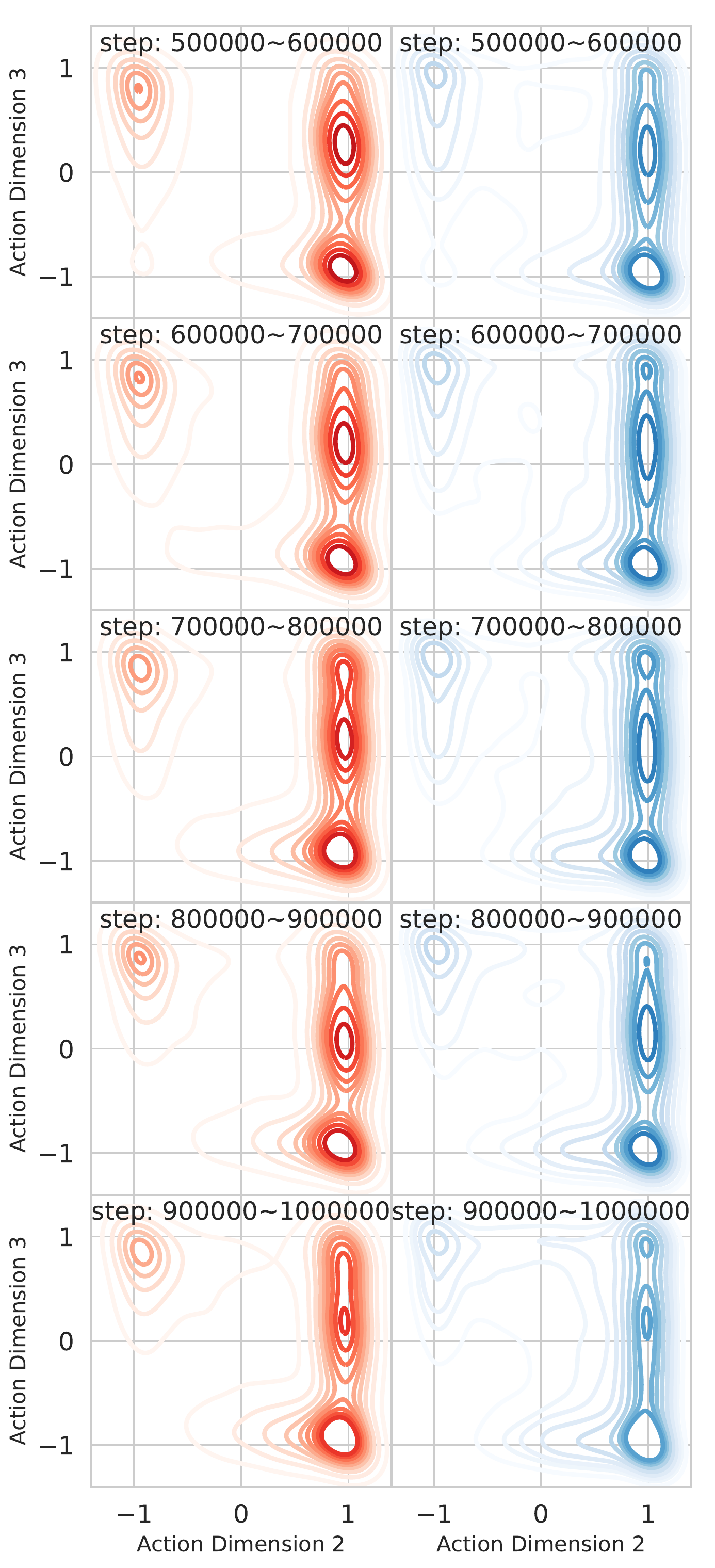}
         \label{fig:12b}
     \end{subfigure}%
     \hfill
    \caption{The second and third dimensional distribution of the action in Walker2d-v3.
    The red contour represents the current policy’s actions and the blue contour represents the evolutionary actions.}
    \label{fig:12}
\end{figure*}

\begin{figure*}[htb]
     \centering
     \begin{subfigure}[t]{0.50\textwidth}
         \centering
        \includegraphics[width=0.95\textwidth]{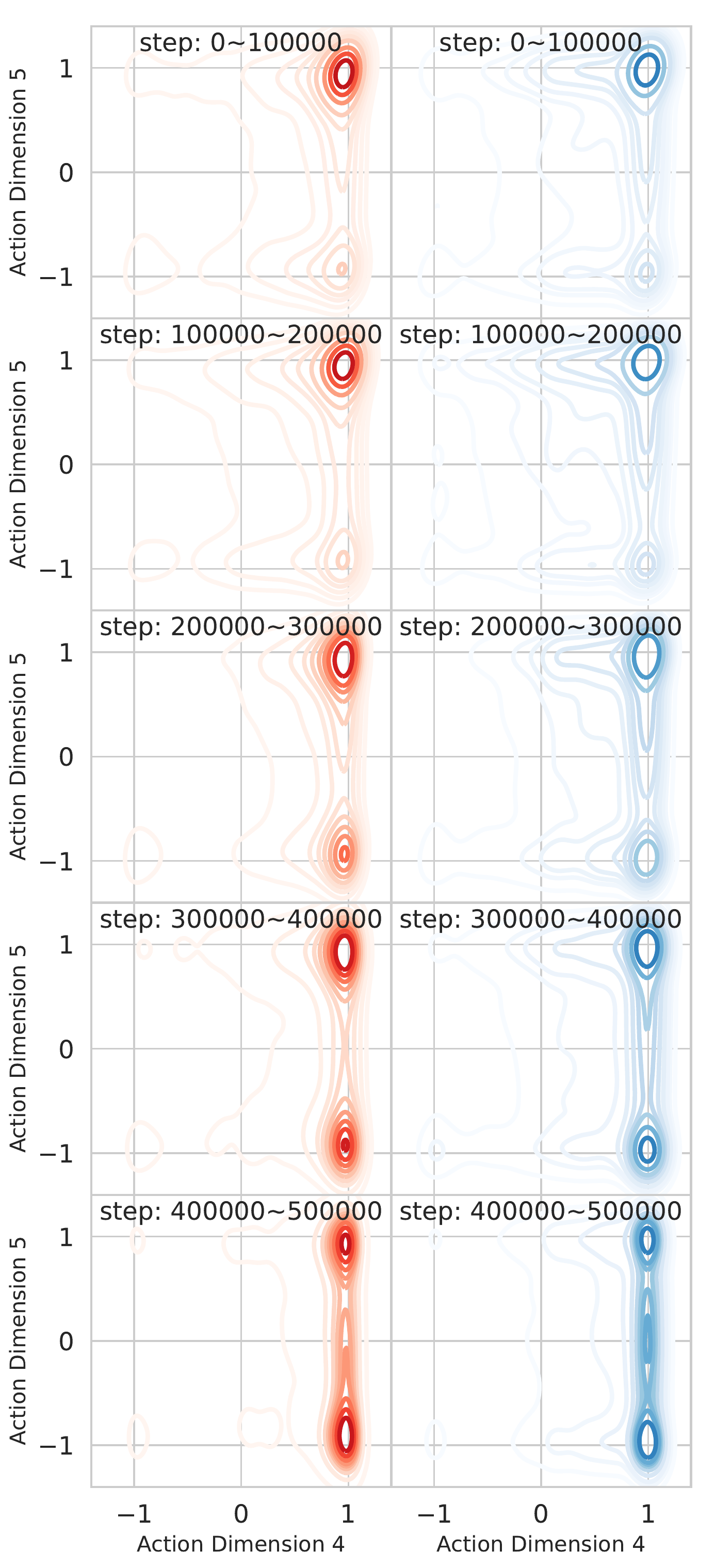}
         \label{fig:13a}
     \end{subfigure}%
     \hfill
     \begin{subfigure}[t]{0.50\textwidth}
         \centering
         \includegraphics[width=0.95\textwidth]{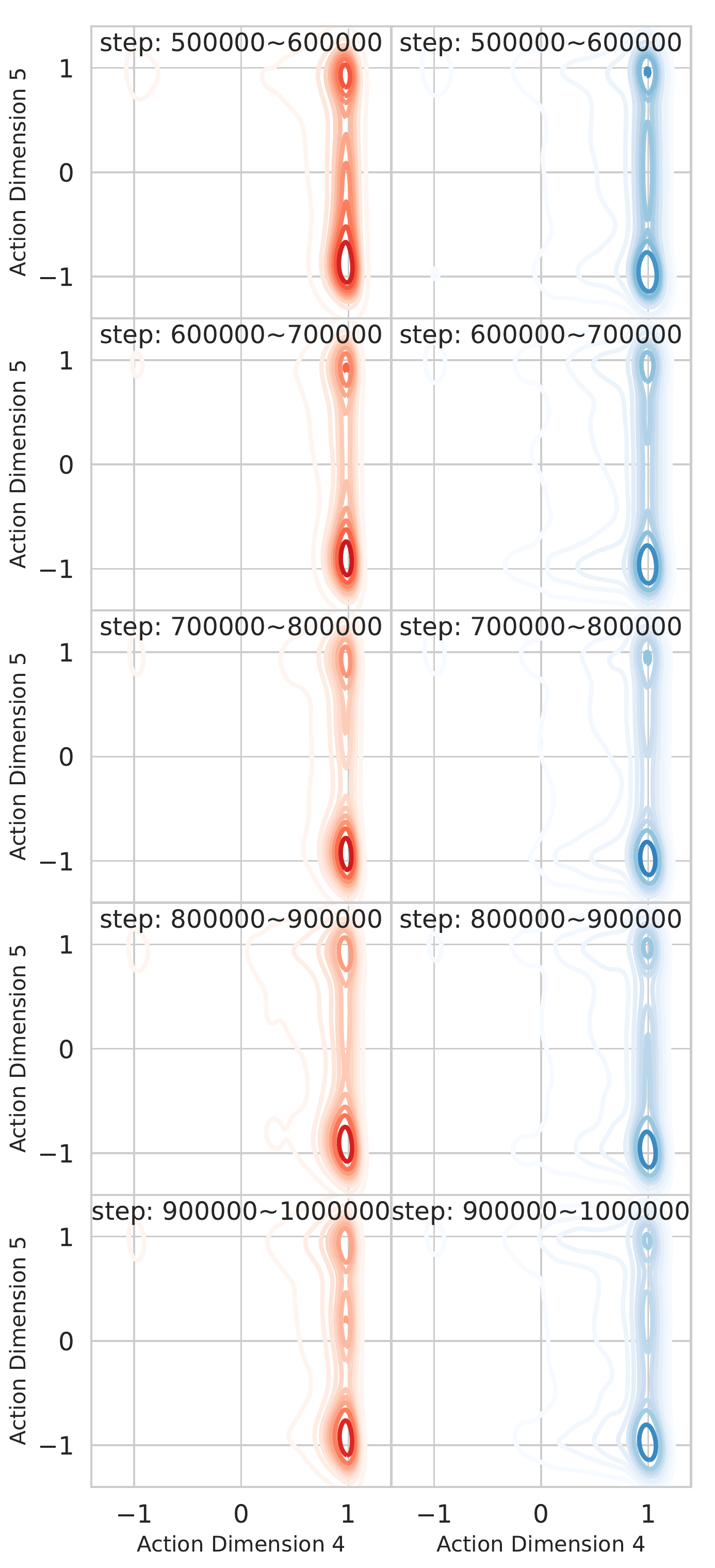}
         \label{fig:13b}
     \end{subfigure}%
     \hfill
    \caption{The fourth and fifth dimensional distribution of the action in Walker2d-v3.
    The red contour represents the current policy’s actions and the blue contour represents the evolutionary actions.}
    \label{fig:13}
\end{figure*}

\end{document}